\documentclass[11pt,reqno]{amsart}

\usepackage{amsmath}
\usepackage{amsthm}
\usepackage{graphicx}
\usepackage{mathtools}
\usepackage{amsfonts}
\usepackage{amssymb}
\usepackage{hyperref}
\usepackage{bm}
\usepackage[margin=1 in]{geometry}
\usepackage{enumerate}
\usepackage[normalem]{ulem} 
\usepackage{tikz}
\usepackage{xcolor}
\usetikzlibrary{matrix,arrows,positioning,automata}
  \usepackage{cancel}
\usepackage{amsmath,centernot}
\usepackage{mathrsfs}  

\DeclareFontFamily{U}{mathx}{}
\DeclareFontShape{U}{mathx}{m}{n}{<-> mathx10}{}
\DeclareSymbolFont{mathx}{U}{mathx}{m}{n}
\DeclareMathAccent{\widecheck}{0}{mathx}{"71}

\numberwithin{equation}{section}

\newtheorem{theorem}{Theorem}[section]
\newtheorem{lemma}[theorem]{Lemma}
\newtheorem{proposition}[theorem]{Proposition}

\theoremstyle{definition}

\newtheorem{remark}[theorem]{Remark}

\def\E{{\mathbb E}}

\def\N{{\mathbb N}}
\def\P{{\mathcal P}}

\def\R{{\mathbb R}}

\def\W{{\mathcal W}}
\def\X{{\mathcal X}}

\title[A hierarchical entropy method for the delocalization of bias]{A hierarchical entropy method for the delocalization of bias in high-dimensional Langevin Monte Carlo}
\author{Daniel Lacker and Fuzhong Zhou} 
\address{Department of Industrial Engineering \& Operations Research, Columbia University}
\email{daniel.lacker@columbia.edu, fz2329@columbia.edu} 
\thanks{D.L.\ acknowledges support from an Alfred P.\ Sloan Fellowship and the NSF CAREER award DMS-2045328.}

\begin{document}

\begin{abstract}
The unadjusted Langevin algorithm is widely used for sampling from complex high-dimensional
distributions. It is well known to be biased, with the bias typically scaling linearly with the dimension when measured in squared Wasserstein distance. However, the  recent paper of Chen et al.\ \cite{DelocalizationOfBias} 
identifies an intriguing new delocalization effect: For a class of distributions with sparse interactions,
the bias between low-dimensional marginals scales only with the lower dimension, not the
full dimension. In this work, we strengthen the results of \cite{DelocalizationOfBias}  in the sparse interaction regime
by removing a logarithmic factor, measuring distance in relative entropy (a.k.a.\ KL-divergence), and relaxing the
strong log-concavity assumption. In addition, we expand the scope of the delocalization
phenomenon by showing that it holds for a class of distributions with weak interactions. Our proofs are based on a hierarchical analysis of the marginal relative entropies, inspired by the authors' recent work on propagation of chaos.
\end{abstract}

\maketitle

\section{Introduction}

Langevin Monte Carlo (LMC), also known as the unadjusted Langevin algorithm, is a simple but widely used sampling method, especially in high dimension. Given a probability density $\pi(x) \propto e^{-V(x)}$ on $\R^n$ and a step size $h > 0$, the LMC iteration is given by
\begin{equation}
X_{(k+1)h} = X_{kh} - h\nabla V(X_{kh}) - \sqrt{2h} Z_k,  \label{intro:LMCiterates}
\end{equation}
where $(Z_k)$ are iid standard Gaussians.
There is a rich literature on theoretical convergence guarantees for LMC, especially in the benchmark strongly log-concave and smooth setting.
For fixed $h > 0$, $X_{kh}$ converges to a distribution $\pi_h$ as $k\to\infty$, and the convergence rate is typically dimension-free, when measured in relative entropy (a.k.a.\ Kullback-Leibler divergence) or the (quadratic) Wasserstein metric.

However, the distance between $\pi_h$ and $\pi$, known as the \emph{bias}, inevitably depends on the ambient dimension $n$. The best known bound on the squared Wasserstein distance $W_2^2(\pi_h,\pi)$ or the relative entropy $H(\pi_h\,|\,\pi)$ is $O(nh)$ (or $O(nh^2)$ under higher-order smoothness assumptions).
This suggests an unfortunate tradeoff between accuracy and step-size: To achieve a given level of accuracy, the step size must scale inversely with the dimension, $h \sim 1/n$.
This tradeoff is often ignored in practice, in extremely high-dimensional applications such as molecular dynamics for which accordingly small step sizes are impractical.

The recent paper \cite{DelocalizationOfBias} by Chen, Cheng, Niles-Weed, and Weare identifies a remarkable new phenomenon which can justify ignoring this tradeoff, which they term \emph{delocalization of bias}: The bias contained in low-dimensional marginals often scales only with the lower dimension. In other words, the bias-per-coordinate is often dimension-free. To state this precisely, we use the following notation. For a probability measure $\rho$ on $\R^n$ and a set $u \subset [n]=\{1,\ldots,n\}$, let us write $\rho^u$ for the marginal of the coordinates in $u$. That is, if $(X_i)_{i\in [n]} \sim \rho$, then $(X_i)_{i \in u} \sim \rho^u$. The paper \cite{DelocalizationOfBias} shows in certain contexts that delocalization of bias holds in the sense that
\begin{equation}
\max_{u \subset [n], \, |u|=k} W_2^2(\pi^u_h,\pi^u) = O( h k\log n). \label{intro:delocalization-CCNW}
\end{equation}
When this holds, it means roughly that a $k$-dimensional observable can be accurately estimated using a step size of order merely $h \sim 1/k\log n$, rather than $h \sim 1/n$. We refer to the introduction of \cite{DelocalizationOfBias} for a more thorough explanation of the practical implications of this phenomenon, as well as a justification for focusing on LMC as opposed to other sampling methods.

The authors of \cite{DelocalizationOfBias} rigorously prove  \eqref{intro:delocalization-CCNW} in two particular contexts, namely when $\pi$ is Gaussian or when the potential $V$ satisfies a certain structural condition of \emph{sparse interactions} which will be described in more detail below. Note also that \eqref{intro:delocalization-CCNW} holds trivially for product measures as well, without the log factor.
However, delocalization is not a universal phenomenon. As illustrated in \cite[Example 1.4]{DelocalizationOfBias}, it can fail even for simple examples, such as a rotation of a strongly log-concave product measure in which a single coordinate becomes, in a sense, too influential.

The goal of this paper is to deepen the study of delocalization of bias, in the following ways:
\begin{enumerate}
\item We improve the main result of \cite{DelocalizationOfBias} in the sparse setting. Specifically, we replace $W_2^2$ with relative entropy (a.k.a.\ Kullback-Leibler divergence), which is stronger by Talagrand's inequality. We also remove the logarithmic factor, improving \eqref{intro:delocalization-CCNW} to the optimal $O(hk)$.
\item We enlarge the scope of the delocalization phenomenon, by proving that it holds for a different class of potentials with \emph{weak interactions}. This class neither contains nor is contained in the sparse setting. To summarize very roughly, the sparseness condition constrains the pattern of zeros in the Hessian $\nabla^2 V$, whereas the weak interaction condition constrains the $\ell^\infty\to\ell^\infty$ operator norm of $\nabla^2 V$. 
\item We do not require $\pi$ to be log-concave, as in \cite{DelocalizationOfBias}. Instead, we make two assumptions which are both implied by strong log-concavity: that $\pi$ satisfies a logarithmic-Sobolev inequality, and also that its conditional measures  satisfy certain transport inequalities.
\end{enumerate}

Our method of proof combines ideas from the Vemapala-Wibisono approach for LMC analysis \cite{VempalaWibisono} with methods introduced in a prior paper of ours \cite{LackerYeungZhou} in the study of propagation of chaos. The former paper \cite{VempalaWibisono} analyzes LMC by using a continuous-time interpolation of the LMC iterates and then analyzing the time-derivative of relative entropy (relative to the target $\pi$) along this interpolation. The latter paper \cite{LackerYeungZhou} builds on prior work of the first author \cite{Lacker-hierarchies,LacLeFlem} on \emph{mean field approximations}, which on the surface has nothing to do with LMC.
 As the present work will illustrate, the methods of \cite{LackerYeungZhou} are well-suited to the more general problem of quantitatively comparing low-dimensional marginals of high-dimensional Fokker-Planck equations. The core idea of the approach is to show that the set-indexed vector $(H_{kh}(u))_{u \subset [n]}$  satisfies a system of linear inequalities, recursive in $k$, where $H_{kh}(u)$ denotes the relative entropy between the law of $(X_{kh}^i)_{i \in u}$ and $\pi^u$, for each coordinate set $u \subset [n]$.

 \subsection{Main results for sparse interactions} \label{se:mainresults-sparse}

The sparse setting is built on a simple undirected graph $G=([n],E)$ which encodes the interactions between coordinates. The edge set is defined by
\[
E := \{(i,j) \in [n]^2 : \partial_{ij} V \text{ is not identically zero} \}.
\]
For $i \in [n]$, let $N_0(i)=\{i\}$, and the $k$-neighborhood $N_k(i)$ of vertex $i$ is defined recursively as the union of $N_{k-1}(i)$ with the set of vertices in $[n] \setminus N_{k-1}(i)$ having an edge to a vertex in $N_{k-1}(i)$. Let us write also $N_k(u)=\cup_{i \in u}N_k(i)$ for $u \subset [n]$.
Abbreviate $N(i)=N_1(i)$ and $N(u)=N_1(u)$. 
For $x=(x_i)_{i \in [n]} \in \R^n$ and $u \subset [n]$ we write $x_u=(x_i)_{i \in u}$. 
In particular, note that $\partial_iV(x)$ depends on $x$ only through $x_{N(i)}$.

The sparse setup introduced \cite{DelocalizationOfBias} is slightly different. They define a graph first, and then assume the potential has the structure $V(x) = \sum_{i=1}^n V_i(x_{N(i)})$ for some functions $V_i : \R^{N(i)}\to\R$. 
Instead, we take the potential $V$ as given, and we define the graph in terms of the potential.
These two perspectives are essentially equivalent, up to changing constants in the neighborhood growth assumptions below.

The main assumptions in \cite[Theorem 2.4]{DelocalizationOfBias} are that $V$ is strongly convex and smooth, i.e., $0 < \alpha I \le \nabla^2 V \le \beta I < \infty$, and that the graph has polynomial growth, in the sense that there exist $p \in \N$ and $c > 0$ such that
\begin{equation}
\max_{i \in [n]} |N_k(i)| \le c(1+k)^p , \ \text{ for all } k \in \N. \label{intro:def:sparse}
\end{equation}
Under these assumptions, \cite[Theorem 2.4]{DelocalizationOfBias} shows for $h\le \alpha/\beta^2$ that
\begin{equation}
\W_{2,\ell^\infty}(\pi_h,\pi) \lesssim h\log(2n) \min\bigg\{\Big( \frac{\beta}{\alpha}\log(2n)\Big)^{p+2}, \ \frac{\beta^2}{\alpha^2}n \Big)\bigg\}. \label{intro:prevresult}
\end{equation}
Throughout the paper, we will use the following Wasserstein distances, defined for any $p \ge 1$ and any probability measures $\mu,\nu$ on a common Euclidean space with finite $p$th moments:
\begin{align*}
W_p^p(\mu,\nu) &:= \inf_\pi \int |x-y|^p\,\pi(dx,dy), \qquad 
W_{p,\ell^\infty}^p(\mu,\nu) := \inf_\pi \int |x-y|_\infty^p\,\pi(dx,dy),
\end{align*}
where the infima are over couplings of $\mu$ and $\nu$, and $|\cdot|$ and $|\cdot|_\infty$ denote the $\ell^2$ (Euclidean) norm and $\ell^\infty$ norm, respectively.
The estimate \eqref{intro:prevresult} controls all low-dimensional marginals via
\begin{equation}
W_2^2(\pi_h^u,\pi^u) \lesssim  h|u| \log(2n) \min\bigg\{\Big( \frac{\beta}{\alpha}\log(2n)\Big)^{p+2}, \ \frac{\beta^2}{\alpha^2}n \Big)\bigg\}, \quad \forall u \subset [n]. \label{intro:prevresult2}
\end{equation}

Our first two theorems improve this bound in several ways. We will work extensively with the relative entropy $H$ (a.k.a.\ Kullback-Leibler divergence) and relative Fisher information $I$, defined for two probability measures on Euclidean space by
\begin{equation}
H(\nu\,|\,\mu) := \int f\,\log f \,d\mu, \qquad I(\nu\,|\,\mu) := \int |\nabla \log f|^2\,d\mu, \qquad f := \frac{d\nu}{d\mu}. \label{def:H-I}
\end{equation}
We set $H(\nu\,|\,\mu) :=\infty$ if $\nu\not\ll\mu$, and similarly $I(\nu\,|\,\mu) :=\infty$ if $\nu\not\ll\mu$ or if $\log f$ is not weakly differentiable.
We make the following assumptions:
\begin{enumerate}
\item[(LSI)] Logarithmic Sobolev inequality: There exists $\alpha > 0$ such that $\pi(dx)\propto e^{-V(x)}dx$ satisfies 
\[
2\alpha H(\mu\,|\,\pi) \le I(\mu\,|\,\pi), \quad \forall \mu \in \P(\R^n).
\]
\item[(A.1)] Lipschitz gradient: There exists $\beta > 0$ such that $\nabla V$ is $\beta$-Lipschitz.
\item[(A.2)] Conditional Talagrand inequality: There exists $\gamma > 0$ such that 
\[
W_1^2(\mu,\pi^{N(u)\setminus u|u}_{x_u}) \le \frac{2\gamma}{\alpha}H(\mu\,|\,\pi^{N(u)\setminus u|u}_{x_u}), \quad \forall u \subset [n], \ x_u \in \R^u, \ \mu \in \P(\R^{N(u)\setminus u}),
\]
where $\pi^{N(u)\setminus u|u}_{x_u}$ denotes the conditional law of $X_{N(u) \setminus u}$ given $X_u=x_u$ under $X \sim \pi$. 
\end{enumerate}
If $\pi$ is strongly log-concave in the sense that $\nabla^2 V \ge \alpha I > 0$, then (LSI) holds by Bakry-\'Emery \cite[Theorem 2]{otto2000generalization}. Moreover, (A.2) holds with $\gamma=1$, because strongly log-concave measures satisfy the Talagrand inequality, and because strong log-concavity (with the same parameter) is preserved by the operations of conditioning, which is easy to show, and marginalizing, which is a well known consequence of the Pr\'ekopa-Leindler inequality.
These assumptions remain valid for bounded perturbations of strongly log-concave measures, by Holley-Stroock and the fact that LSI implies Talagrand inequality \cite[Theorem 1]{otto2000generalization}.

Our first theorem assumes polynomial growth of the graph's neighborhoods.

\begin{theorem} \label{th:sparse-poly}
Suppose assumptions (LSI) and (A.1,2) hold, and there exist $c,p \ge 1$ such that
\begin{equation}
\max_{i \in [n]}|N_{k+1}(i)| \le c(1+k^p), \quad  \forall k \ge 0. \label{asmp:polygrowth}
\end{equation}
Then there exist constants $C,h^*>0$, depending only on $(\alpha,\beta,\gamma,c,p)$, such that
\begin{equation}
\frac{\alpha}{2}W_2^2(\pi^u_h,\pi^u)  \le H(\pi^u_h\,|\,\pi^u) \le Ch|u|, \qquad \forall u \subset [n], \ \ 0 \le h \le h^*. \label{thm:mainbound}
\end{equation}
We may explicitly take
\begin{equation*}
C = 40c^2p^p\frac{\beta^2}{\alpha}\Big( \frac{8\gamma\beta^2}{\alpha^2} + 1\Big)^p, \qquad h^* = \frac{\alpha}{4c\beta^2}.
\end{equation*}
\end{theorem}

Theorem \ref{th:sparse-poly} is generalized by Theorem \ref{th:sparse-dynamic-poly}, which includes estimates of the LMC dynamics, not just at stationarity.
Ignoring dependence on other parameters, the bound on the bias in Theorem \ref{th:sparse-poly} is $O(kh)$ in any coordinate set $u$ of size $k$. The ambient dimension $n$ plays no direct role in this bound or in the restriction on the time step $h$.
In particular, Theorem \ref{th:sparse-dynamic-poly} has no logarithmic factor like \eqref{intro:prevresult2}. However, the dependence on the condition number $\beta/\alpha$ is slightly better in \eqref{intro:prevresult2} where it is of order $(\beta/\alpha)^{p+2}$, in contrast with $(\beta/\alpha)^{2p+2}$ in our Theorem \ref{th:sparse-poly}. We have not attempted to optimize the dependence on $c$ or $p$ in \eqref{thm:mainbound}, other than in the exponent of $\beta/\alpha$.

Our next theorem, still in the sparse regime, pushes the growth condition as far as our methods seem to allow, from polynomial to exponential, up to a certain critical exponent.

\begin{theorem} \label{th:sparse-exp}
Suppose assumptions (LSI) and (A.1,2) hold, and there exist $c,r \ge 1$ such that
\begin{equation}
r < 1 + \frac{\alpha^2}{\gamma\beta^2}, \qquad \max_{i\in [n]}|N_{k+1}(i)| \le cr^k , \quad \forall k \ge 0. \label{asmp:expgrowth}
\end{equation}
Then there exist constants $C,h^*>0$, depending only on $(\alpha,\beta,\gamma,c,r)$ such that \eqref{thm:mainbound} holds.
Setting $\eta := 1-\frac{\gamma\beta^2}{\alpha^2}(r-1)$, and noting that $0 < \eta < 1$, we may explicitly take
\begin{equation*}
C = \frac{10\beta^2c^2}{\alpha\eta^3 } e^{\gamma(r-1)/c}, \qquad h^* =  \frac{\alpha  \eta^{3/2}}{5\beta^2c  } .  
\end{equation*}
\end{theorem}

We do not know if the condition $r < 1 + \frac{\alpha^2}{\gamma\beta^2}$ is sharp. 
Of course, the assumption \eqref{asmp:expgrowth} is more general than \eqref{asmp:polygrowth}, and if one is interested only in how the bias depends on $|u|$ and $h$ then Theorem \ref{th:sparse-poly} follows as a special case of Theorem \ref{th:sparse-exp}. We state them separately because the dependence on the condition number $\beta/\alpha$ is somewhat unclear in Theorem \ref{th:sparse-exp}, due to its appearance in both $C$ and the constraint on $r$ in \eqref{asmp:expgrowth}.  In contrast, Theorem \ref{th:sparse-poly} has a simple dependence on $\beta/\alpha$, which we derive directly rather than as a consequence of Theorem \ref{th:sparse-exp}. 
This dependence deteriorates as the growth rate $p$ of the graph increases, and we might interpret the extreme case of Theorem \ref{th:sparse-exp} as a bound of the form $f(\beta/\alpha)$ where $f(x)$ jumps to $+\infty$ as $x$ increases past the critical value of $1/\gamma(r-1)$.

We lastly record a similar delocalization phenomenon for the continuous-time Langevin dynamics, which we prove in Section \ref{se:continuoustime} via a much simpler version of our proof of Theorem \ref{th:sparse-poly}. 

\begin{theorem} \label{th:continuoustime}
Suppose assumptions (LSI) and (A.1,2) hold.
Suppose $(Y_t)_{t \ge 0}$ solves 
\[
dY_t = -\nabla V(Y_t)dt + \sqrt{2}dB_t, \qquad Y_0 \sim \rho_0,
\]
where $\rho_0$ is some given initial distribution. Let $\rho_t$ denote the law of $Y_t$ for $t > 0$.
Let $H_t(u) = H(\rho^u_t\,|\,\pi^u)$ for each $t \ge 0$ and $u \subset [n]$. Let $(\Lambda(t))_{t \ge 0}$ denote a Poisson process with rate $\gamma\beta^2 /2\alpha$. Then
\begin{equation}
H_t(u) \le e^{-2\alpha(1-\epsilon)t} \E H_0\big(N_{\Lambda(t/\epsilon)}(u)\big), \qquad \forall u \subset [n], \ t \ge 0, \ \epsilon \in (0,1). \label{ineq:continuoustime}
\end{equation}
\end{theorem}

The Poisson process also appears behind the scenes in the proofs of Theorems \ref{th:sparse-poly} and \ref{th:sparse-exp}.
In \eqref{ineq:continuoustime} we may always bound $H_0(N_k(u)) \le H_0([n])$ for all $k$, by the data processing inequality, and send $\epsilon \to 0$ to recover the usual ``global'' entropy estimates $H_t([n]) \le e^{-2\alpha t}H_0([n])$. Note that one cannot expect a bound like $H_t(u) \le e^{-ct}H_0(u)$ for all $u$; for example, if $\rho^u_0=\pi^u$ but $\rho_0 \neq \pi$ then there is no reason to expect that $\rho^u_t=\pi^u$ for $t > 0$. But in situations where $H_0(N_k(u))$ grows slowly enough with respect to $k$, the bound \eqref{ineq:continuoustime} reveals a form of the delocalization phenomenon. For instance, if $H_0(u) \le C_0|u|$ for all $ u \subset [n]$ and $|N_k(u)| \le c(1+k^p)$ for all $k \ge 0$ as in Theorem \ref{th:sparse-poly}, then
\begin{equation*}
\E H_0(N_{\Lambda(s)}(u)) \le cC_0(1+\E\Lambda^p(s))|u|, \quad s \ge 0, 
\end{equation*}
and by taking $\epsilon=1/\max(1,2\alpha t)$ we can extract from \eqref{ineq:continuoustime} a bound like $H_t(u) = O(e^{-2\alpha t}|u|)$.

\subsection{Main results for weak interactions} \label{se:mainresults-weak}

We next introduce a new class of potentials for which delocalization of bias can be shown, which we call \emph{weak interactions}. 
We give two results.
We start with the one that is simpler to state, based on the framework of \cite{DelocalizationOfBias}, which was kindly suggested  (along with its proof) by Yifan Chen after a discussion of our Theorem \ref{th:weak}.  In the following, we write $\|M\|_{\infty\to\infty}$ for the $\ell^\infty$-operator norm of an $n \times n$ matrix $M$.

\begin{theorem} \label{th:onestepcontraction}
Let $\beta > \alpha > \alpha_0 > 0$. Suppose $\alpha I \le \nabla^2 V \le \beta I$. Assume that for each $x \in \R^n$ there is a decomposition $\nabla^2V(x)=D(x)+M(x)$, where $D(x)$ is a diagonal matrix, and $M(x)$ is zero on the diagonal and satisfies $\|M(x)\|_{\infty \to \infty} \le \alpha_0$.  Then, for $h \le 1/\beta$,
\[
W^2_{2,\ell^\infty}(\pi_h,\pi) \le  h\log(2n)\Big(\frac{4\beta}{\alpha-\alpha_0} \Big)^2 .
\]
In particular,
\[
W^2_2(\pi_h^u,\pi^u) \le |u|h\log(2n)\Big(\frac{4\beta}{\alpha-\alpha_0} \Big)^2  , \qquad \forall u \subset [n].
\]
\end{theorem}

The proof of Theorem \ref{th:onestepcontraction}, given in Section \ref{se:onestepcontration}, is an application of the method of \cite{DelocalizationOfBias}, and a particularly simple one because we in fact have a one-step contraction along the LMC iterates. Their more complex multi-step analysis, introduced for the sparse case, is not needed.

A natural class of examples arises from pairwise interaction potentials, of the form
\begin{equation}
V(x) = \sum_{i=1}^n V_i(x_i) + \sum_{1 \le i < j \le n}^n V_{ij}(x_i-x_j), \label{intro:def:pairwise-nonexch}
\end{equation}
for some functions $V_i$ and $V_{ij}$ on $\R$, with $V_{ij}$ assumed (for simplicity) to be even and satisfying $V_{ij}=V_{ji}$.
The assumptions of Theorem \ref{th:onestepcontraction} hold for suitable $\beta$ if $(V''_i,V''_{ij})$ are bounded from above  uniformly over $(i,j)$, if $V_i'' \ge \alpha$, if $V_{ij}'' \ge 0$, and if $\max_i\sum_{j \neq i} |V_{ij}''| \le \alpha_0$. 
As a special case, this includes symmetric \emph{mean field} models, in which $V_i=U$ and $V_{ij}=W/n$ for some $(i,j)$-independent functions $U$ and $W$. However, exchangeable models like this are not so interesting from a delocalization perspective, as will be explained in Section \ref{se:subadditivity}.

We next give a second result on weak interactions, based on our new entropic methods rather than the approach of \cite{DelocalizationOfBias}. The advantages are that it will remove the log factor, relax the strong log-concavity assumption, and hold in relative entropy instead of Wasserstein. The major disadvantage is that it will require a much more specific structural assumption on the potential $V$, which is more restrictive than the decomposition in Theorem \ref{th:onestepcontraction}.
Precisely, the potential takes the form
\begin{equation}
V(x) = \sum_{ u \subset [n]} V_u(x_u). \label{def:Vstructure}
\end{equation}
Here the sum is over nonempty subsets of $[n]$, and $V_u : \R^u \to \R$ are some given functions.
The following parameters $(M_0,M_1,R_1)$ will play important roles:
\begin{equation}
M_p := \max_{i \in [n]}\sum_{w \ni i} L_w|w|^p, \qquad R_p := \max_{i \in [n]}\sum_{w \ni i, \, |w| \ge 2}L_w(|w|-1)^p, \quad \text{for } p =0,1. \label{def:MpRp-constants}
\end{equation}
In addition to (LSI), we will make the following assumptions:
\begin{enumerate}
\item[(B.1)] Lipschitz gradients: For each set $u \subset [n]$ there is a constant $L_u \ge 0$ such that $\nabla V_u$ is $L_u$-Lipschitz.
\item[(B.2)] Conditional Talagrand inequality: There exists $\gamma > 0$ such that 
\[
W_1^2(\mu,\pi^{w \setminus u|u}_{x_u}) \le \frac{2\gamma}{\alpha}H(\mu\,|\,\pi^{w\setminus u|u}_{x_u}), \quad \forall u,w \subset [n], \ L_w \neq 0, \ x_u \in \R^u, \ \mu \in \P(\R^{w\setminus u}),
\]
where $\pi^{w\setminus u|u}_{x_u}$ denotes the conditional law of $X_{w \setminus u}$ given $X_u=x_u$ under $X \sim \pi$.
\item[(B.3)] Weak interactions: $\gamma M_0R_1 < \alpha^2$.
\end{enumerate}
We think of $R_p$ as a measure of the maximal \emph{interaction} effect felt by any coordinate, because the sum excludes singletons $w$. 
 The assumption (B.3) thus requires that the interactions are sufficiently weak, a condition which is common in the literature on uniform-in-time mean field limits \cite{LacLeFlem}.

\begin{theorem} \label{th:weak}
Under assumptions (LSI) and (B.1--3), there exist constants $C,h^*>0$, depending only on $(\alpha,M_0,M_1,R_1,\gamma)$, such that \eqref{thm:mainbound} holds.
Setting $\eta :=  1 - \frac{\gamma M_0 R_1}{\alpha^2}$, and noting that $0 < \eta < 1$, we may explicitly take
\begin{equation*}
C = \frac{10M_0 M_1 }{\alpha\eta^3 }e^{\gamma }, \qquad h^* = \frac{\alpha  \eta^{3/2}}{5 M_0M_1 } . 
\end{equation*}
\end{theorem}

We will see in \eqref{ineq:Lipschitz-weak} that the Lipschitz constant of $\nabla V$ is bounded by $M_0$, which is in turn bounded by $M_1$. We thus think of $M_0M_1/\alpha^2$ as a weaker form of the squared condition number.
The constant $R_1$ play a different role, measuring only the strength of \emph{interations} between coordinates, and it notably shows up only in the smallness condition (B.3) which is equivalent to $\eta > 0$.

Theorem \ref{th:weak} is generalized by Theorem \ref{th:weak-dynamic}, which includes dynamic estimates.
The structural assumption \eqref{def:Vstructure} is meant as a generalization of the pairwise interaction setting of \eqref{intro:def:pairwise-nonexch}, allowing interactions between arbitrarily many coordinates (three-way, four-way, etc.). However, assumption (B.3) constrains how strong the higher-order interactions are allowed to be. Essentially, in order to keep $R_1$ small, the weights $L_w$ must be smaller for larger sets $w$.
The pairwise case fits into the setting of Theorem \ref{th:weak} by taking $V_u=0$ for all $|u| \ge 2$, and we then have 
\begin{align*}
M_0 &= \max_{i \in [n]} \|V''_i\|_\infty + \sum_{j \neq i} \|V_{ij}''\|_\infty, \qquad R_1 = \max_{i \in [n]} \sum_{j \neq i} \|V_{ij}''\|_\infty.
\end{align*}
In the log-concave setting $\nabla^2 V \ge \alpha I$, we can again take $\gamma=1$ in (B.2), and then assumption (B.3) is stronger than the decomposition assumption of Theorem \ref{th:onestepcontraction}.

The weak interaction setting of Theorems \ref{th:onestepcontraction} and \ref{th:weak} is quite different from the sparse setting of Theorems \ref{th:sparse-poly} and \ref{th:sparse-exp}. For instance, in the case of pairwise interactions \eqref{intro:def:pairwise-nonexch}, the weak interaction imposes constraints on the  \emph{row sums} $\sum_j \|V''_{ij}\|_\infty$. Sparseness, on the other hand, constraints the \emph{zero pattern} of the matrix $(V_{ij}'')_{i,j \in [n]}$. These regimes are not directly comparable. An analog of Theorem \ref{th:continuoustime} is possible for continuous-time Langevin under the weak interaction assumptions, where a more complicated $2^{[n]}$-valued ``growth process'' takes the place of the Poisson-driven process $(N_{\Lambda(t)}(u))_{t \ge 0}$. We omit the details, as it will become apparent from the proofs of Theorems \ref{th:continuoustime} and \ref{th:weak} how to concoct such a result.

It is worth noting that none of the above results on sparse and weak interactions covers the general Gaussian case, which was studied directly in \cite[Example 1.3]{DelocalizationOfBias}. An interesting question posed in \cite{DelocalizationOfBias} is to identify a broader class of models for which delocalization holds which includes both the Gaussian and sparse cases. The case of weak interactions is a third class to add to the mix, and it raises the natural and more challenging question of finding an even more general delocalization result which includes all three classes.

\subsection{Related literature and methods}

We are not aware of any other work on the delocalization of bias phenomenon since its recent introduction in \cite{DelocalizationOfBias}. 
Perhaps the most closely related is the recent paper \cite{CuiLiuTong} which developed an interesting approach based on Stein's method for comparing low-dimensional marginals, motivated by sampling problems for spatial models.

From the sampling literature, the main point of departure for our methodology is the seminal work of Vempala-Wibisono \cite{VempalaWibisono}, and we refer also to \cite[Section 4.2]{ChewiBook} for a nice textbook treatment. Their approach was based on time-differentiating the relative entropy along a Fokker-Planck equation satisfied by the time-marginal laws of a continuous-time interpolation of the LMC iterates. Their approach  also works with R\'enyi entropies, a generalization we do not pursue here. We deepen their method by extending it to each marginal $u \subset [n]$. This leads to a recursive linear system of differential inequalities for the set-indexed vector of marginal entropies. This system of inequalities is hierarchical, in the sense that the $u$-coordinate depends on the $w$-coordinate only for $w \supset u$.

This hierarchical entropy method draws its inspiration from the first author's work \cite{Lacker-hierarchies} on mean field particle systems with pairwise interactions, where a similar approach related to the so-called \emph{BBGKY hierarchy} was developed in order to derive the sharp rate of propagation of chaos. These mean field models are exchangeable, so the $u$-marginal for a coordinate set $u$ depends on $u$ only through its cardinality. As a result, instead of a vector of entropies indexed by $2^{[n]}$, the paper \cite{Lacker-hierarchies} faced a vector indexed merely by $[n]$, which was substantially simpler to analyze. More recently, our paper \cite{LackerYeungZhou} joint with L.C.\ Yeung adapted the methods of \cite{Lacker-hierarchies} far beyond the mean field setting, deriving non-asymptotic propagation of chaos bounds which are essentially measuring the distance between low-dimensional marginals of an SDE and suitably chosen product measures. To understand the intricate set-indexed hierarchy of differential inequalities, a key idea was introduced in \cite{LackerYeungZhou} which we borrow in this paper: The linear operator applied to the vector of entropies happens to have the structure of a rate matrix (or infinitesimal generator) of a certain Markov process on state space $2^{[n]}$. This opens the door to a relatively clean analysis in stochastic and/or operator-theoretic terms, as opposed to the more brute-force iteration performed in \cite{Lacker-hierarchies} which appeared to be impossible to adapt to the set-indexed setting.

\subsection{Subadditivity versus delocalization} \label{se:subadditivity}
It is worth noting that a weaker notion of delocalization holds automatically, in which the max is replaced by an average, thanks to a well known \emph{subadditivity inequality}, whose simple proof is explained after inequality \eqref{ineq:subadditivity-Euclidean} below:
\begin{equation}
{n \choose k}^{-1}\sum_{u \subset [n], \, |u|=k} W_2^2(\pi^u_h,\pi^u) \le \frac{k}{n} W_2^2(\pi_h,\pi). \label{intro:ineq:subadditive}
\end{equation}
(The same holds for relative entropy in place of $W_2^2$.)
That is, the \emph{typical} $k$-marginal distance is no more than its uniform share of the full distance. Hence, using the well known bound $W_2^2(\pi_h,\pi) = O(nh)$, one immediately deduces that the left-hand side of \eqref{intro:ineq:subadditive} is $O(kh)$. This has implications for certain ``averaged" observables. For instance, if $f(x) = (1/n)\sum_i g_i(x_i)$ for 1-Lipschitz functions $g_i$, then Kantorovich duality and \eqref{intro:ineq:subadditive} yield
\[
\int_{\R^n} f\,d(\pi_h-\pi) = \frac{1}{n}\sum_{i=1}^n \int_\R g_i\,d(\pi^{\{i\}}_h -\pi^{\{i\}}) \le \frac{1}{n}\sum_{i=1}^n W_2(\pi^{\{i\}}_h,\pi^{\{i\}}) = O(h),
\]
with no dimension dependence. This observation, which extends readily to U-statistics, appeared also in \cite{durmus2024asymptotic} in the context of a more general study about how to improve dimension-dependence of $W_2^2(\pi_h,\pi)$ under finer information about $\nabla^2V$.
However, this subadditivity method does not say anything about asymmetric observables, e.g., those depending on just a few distinguished coordinates, which require a finer estimate like  \eqref{intro:delocalization-CCNW}.

One exceptional case worth pointing out is when the measure $\pi$ is exchangeable, in which case the marginal $\pi^u$ depends on the set $u$ only through its size. The same is then true of $\pi_h^u$, and there is no difference between taking a maximum or an average on the left-hand side of \eqref{intro:ineq:subadditive}. In other words, delocalization holds for exchangeable measures as a trivial consequence of subadditivity.

\subsection{On the notion of delocalization} 
This short section discusses some general observations about the concept of delocalization but has nothing to do with LMC, and the terminology introduced here will not be used elsewhere in the paper.

The term \emph{delocalization} originates in random matrix theory, though with a somewhat different meaning. A vector $x\in\R^n$ is said to be delocalized if its mass is spread out across its coordinates in some sense. It has been established for many classical random matrix models that eigenvectors are delocalized in this sense, with high probability. In other words, the joint distribution of the entries of an eigenvector put most of its mass on delocalized vectors. 

Let us explain one precise way that these ideas are formulated in random matrix theory, following \cite{ErdosSchleinYau,ShouVanHandel}.
In the following, let $x_u=(x_i)_{i \in u}$ for any set $u \subset [n]$.
For $k \in [n]$ and $r \ge 0$, we say a vector $x \in \R^n$ is $(k,r)$-delocalized if 
\begin{equation}
\max_{u \subset [n], \, |u|=k}|x_u|^2 \le r|x|^2. \label{def:k-r-delocalized}
\end{equation}
Here $|\cdot|$ denotes the usual Euclidean norm.
In words, the maximal squared norm among $k$-element subsets is no more than a fraction $r$ the squared norm of the full vector. Of course, this statement is only meaningful if $r < 1$.
For example, the vector of all ones is $(k,r)$-delocalized for $r=k/n$. On the other extreme, a basis vector in $\R^n$ is not really delocalized at all, as it is $(k,r)$-delocalized only for $r=1$.
By bounding the maximum from below by the average, and noting that
\begin{equation}
{n \choose k}^{-1}\sum_{u \subset [n], \, |u|=k}|x_u|^2 = {n \choose k}^{-1} \sum_{i=1}^n \sum_{u \subset [n], \, |u|=k}x_i^2 \, 1_{i \in u} = \frac{k}{n}|x|^2, \label{ineq:subadditivity-Euclidean}
\end{equation}
we deduce a sort of \emph{speed limit} for delocalization:
If $x \neq 0$ is $(k,r)$-delocalized, we must have $r \ge k/n$. Note also that \eqref{intro:ineq:subadditive} can be deduced from  \eqref{ineq:subadditivity-Euclidean}, by  replacing $x$ with $X-Y$, where $X \sim \pi_h$ and $Y \sim \pi$ are optimally coupled for $W_2^2(\pi_h,\pi)$.

The notion of delocalization studied in the present work and in  \cite{DelocalizationOfBias} is measure-theoretic in nature and quite different from the Euclidean notion \eqref{def:k-r-delocalized}. For a reference measure $Q$ on $\R^n$, we might say that a measure $P$ on $\R^n$ is $(Q,k,r)$-delocalized if
\begin{equation}
\max_{u \subset [n], \, |u|=k} W_2^2(P^u,Q^u) \le rW_2^2(P,Q). \label{def:k-r-delocalized-measure}
\end{equation}
Other choices of distance or divergence instead of $W_2$ would also be reasonable here.
This says that $k$-dimensional marginals of $P$ are closer to those of $Q$ by a factor of $r$ compared to the full joint distributions. 
To appreciate how different the Euclidean notion \eqref{def:k-r-delocalized} is from its measure-theoretic counterpart \eqref{def:k-r-delocalized-measure}, consider the following two simple examples:
\begin{itemize}
\item Take $Q=\delta_0$, and let $P$ be the uniform measure on the set of $n$ standard basis vectors. Then \eqref{def:k-r-delocalized-measure} holds for any $k$ with $r=k/n$. But for any $r < 1$, the support of $P$ does not contain any vector $x$ satisfying \eqref{def:k-r-delocalized}.
\item Let $X_1,\ldots,X_n$ be iid Bernoulli($1/2$). Let $Q$ be their joint distribution, and let $P$ be the joint law of $(X_1,\ldots,X_{n-1},Y_n)$ where $Y_n=\sum_{i \in [n-1]}X_i \mod 2$. Then $P$ and $Q$ are distinct but have the same $(n-1)$-dimensional marginals, so $P$ is $(Q,k,0)$-delocalized for any $k < n$. There is no \emph{speed limit} here as there is for the Euclidean notion ($r \ge k/n$).
\end{itemize}

\subsection{Organization of the rest of the paper}

The rest of the paper is devoted to the proofs of Theorems \ref{th:sparse-poly}, \ref{th:sparse-exp}, \ref{th:onestepcontraction}, and \ref{th:weak}.  Section \ref{se:firststeps} begins with some common steps shared by Theorems \ref{th:sparse-poly}, \ref{th:sparse-exp}, and \ref{th:weak}. Then, Sections \ref{se:proof-sparse} and \ref{se:proof-weak} respectively carry out the proofs, the former for the sparse cases and the latter for the weak interaction case. Lastly, Section \ref{se:onestepcontration} proves Theorem \ref{th:onestepcontraction}.

\subsection*{Acknowledgements}

We are grateful to Jonathan Weare, Jonathan Niles-Weed, Xiaoou Cheng, and especially Yifan Chen for helpful discussions throughout the course of this work. We also thank Sinho Chewi and Matthew Zhang for comments and suggestions on an early draft of the paper.

\section{First steps of the proof} \label{se:firststeps}

In this section we provide the common first steps shared by the proofs of Theorems \ref{th:sparse-poly}, \ref{th:sparse-exp}, and \ref{th:weak}. Let us first recall and summarize some basic notation and facts that we will use repeatedly. For a vector $x=(x_i)_{i \in [n]} \in \R^n$ and a set $u \subset [n]$, we write $x_u=(x_i)_{i \in u}$ or $x^u=(x_i)_{i \in u}$ for the subvector in $\R^u$. For a measure $\rho$ on $\R^n$, we similarly write $\rho^u$ for the corresponding marginal. We will sometimes use the bracket notation for integration, $\langle\rho,f\rangle = \int f\,d\rho$.

The assumption (LSI) has several important consequences. First, it implies the Poincar\'e  and Talagrand inequalities,
\begin{equation}
\langle \pi,f^2\rangle-\langle\pi,f\rangle^2 \le \frac{1}{\alpha}\langle \pi,|\nabla f|^2\rangle,  \qquad W_2^2(\mu,\pi) \le \frac{2}{\alpha}H(\mu\,|\,\pi), \ \ \forall \mu \in \P(\R^n), \label{ineq:PoincareTalagrand}
\end{equation}
the former holding for Lipschitz functions $f$. See \cite[Theorem 1]{otto2000generalization} and \cite[Section 7]{otto2000generalization}, respectively. Note that (LSI) holds also for every marginal $\pi^u$, $u \subset [n]$, and thus so do the inequalities \eqref{ineq:PoincareTalagrand}.
It is worth noting that $\alpha \le \beta$ when (LSI) holds and  $\nabla V$ is $\beta$-Lipschitz; indeed, the covariance matrix of $\pi$ is bounded from above by $(1/\alpha)I$ in semidefinite order by the Poincar\'e inequality, and it is bounded from below by $(1/\beta)I$ by the Cram\'er-Rao lower bound.

Our proofs start from a standard continuous-time interpolation. Fix $h>0$, and let $t_-$ denote the largest element of $\{0,h,2h,3h,\ldots\}$ less than or equal to $t$. Define
\[
X_t = X_{t_-} - (t-t_-)\nabla V(X_{t_-}) + \sqrt{2}(B_t-B_{t_-}), \qquad t > 0.
\]
where $B$ is a Brownian motion, and $X_0$ is given.
In other words, $X$ solves the (path-dependent) SDE
\begin{equation}
dX_t =  -\nabla V(X_{t_-})dt + \sqrt{2}\,dB_t. \label{eq:XSDE}
\end{equation}
Then $(X_0,X_h,X_{2h},\ldots)$ has the same law as the LMC iterates.
Let $\rho_t$ denote the law of $X_t$, for each $t\ge 0$, and note that for $t > 0$ it has a strictly positive density when, e.g., $\nabla V$ is Lipschitz. 
It was observed in \cite{VempalaWibisono} that  $\rho$ solves (weakly) the Fokker-Planck equation
\begin{equation}
\partial_t \rho = \nabla \cdot (b\rho) + \Delta \rho, \quad \text{where} \quad b_t(x) := \E[\nabla V(X_{t_-})\,|\,X_t=x]. \label{eq:VempalaWibisono1}
\end{equation}
We begin with a refinement of this for the marginals of $\rho$. Throughout the paper, we write $\nabla_uV$ and $\Delta_uV$ for the gradient and Laplacian, respectively, with respect to just the coordinates in $u$.

\begin{proposition} \label{pr:marginaldynamics}
Suppose  $V$ is $C^1$ and $\nabla V$ is Lipschitz.
Let $u \subset [n]$ be nonempty. There exists a jointly measurable function $b^u : [0,\infty) \times \R^u \to \R^u$ such that
\[
b^u_t(X^u_t) = \E[\nabla_u V(X_{t_-}) \,|\,X^u_t], \ \ a.s., \ a.e. \ t > 0.
\]
Moreover, $(\rho^u_t)_{t \ge 0}$ is a solution in the sense of distributions of the Fokker-Planck equation
\begin{equation}
\partial_t\rho^u = \nabla_u(b^u\rho^u) + \Delta_u \rho^u. \label{eq:FKPmarginal}
\end{equation}
\end{proposition}
\begin{proof}
The existence of a jointly measurable version of the conditional expectation is shown in \cite[Proposition 5.1]{BrunickShreve}. To derive the Fokker-Planck equation, let $\varphi$ be a smooth test function of compact support. Apply It\^o's formula to $\varphi(X_t^u)$, using the SDE represetation \eqref{eq:XSDE}:
\begin{align*}
\frac{d}{dt}\int_{\R^u} \varphi\,d\rho^u_t = \frac{d}{dt}\E\varphi(X^u_t) &= \E\Big[-\nabla_u V(X_{t_-}) \cdot \nabla_u\varphi(X^u_t) + \Delta_u \varphi(X_t^u)\Big] \\
	&=\E\Big[-b^u_t(X^u_t) \cdot \nabla_u\varphi(X^u_t) + \Delta_u \varphi(X^u_t)\Big],
\end{align*}
with the last step using the definition of conditional expectation. This rewrites as
\[
\frac{d}{dt}\int_{\R^u} \varphi\,d\rho^u_t = \int_{\R^u} \big(-b^u_t \cdot \nabla_u\varphi + \Delta_u \varphi\big)\,d\rho^u_t,
\]
which is exactly the distributional form of the announced Fokker-Planck equation.
\end{proof}

The main theorems (except for Theorem \ref{th:onestepcontraction}) can be understood as stability estimates for the coupled system of partial differential equations satisfied by $(\rho^u)_{u \subset [n]}$. This system is \emph{hierarchical} in the sense that the drift $b^u$ appearing in the equation for $\rho^u$ depends on $(\rho^w)_{w \supset u}$. Recalling the definition of relative entropy in \eqref{def:H-I}, we abbreviate
\[
H_t(u) := H(\rho_t^u\,|\,\pi^u), \qquad u \subset [n].
\]
The Vempala-Wibisono strategy \cite{VempalaWibisono} was to use \eqref{eq:VempalaWibisono1} in order to time-differentiate the full entropy $H_t([n])$, ultimately obtaining a differential inequality for it. We will instead differentiate all of the marginal entropies $H_t(u)$ and ultimately obtain a system of differential inequalities for them.

\begin{proposition} \label{pr:H_t(u)firstbound}
Suppose  $V$ is $C^1$ and $\nabla V$ is Lipschitz, and suppose (LSI) holds.
For each nonempty $u \subset [n]$, $\epsilon \in (0,1)$, and $t > 0$, we have
\begin{equation*}
\frac{d}{dt}H_t(u) \le A^1_t(u) + A^2_t(u)  - \alpha H_t(u),
\end{equation*}
where 
\begin{align*}
A^1_t(u) &:= \frac{1}{2(1-\epsilon)}\E\big|\nabla_u V(X_t)-\nabla_u V(X_{t_-})\big|^2 , \\
A^2_t(u) &:= \frac{1}{2\epsilon}\int_{\R^u} \Big|\E[\nabla_u V(X_t)\,|\,X^u_t=x_u] - \E_{Y \sim \pi}[\nabla_uV(Y)\,|\,Y^u=x_u]\Big|^2  \,\rho^u_t(dx_u).
\end{align*}
Here we adopt the convention that $A^2_t([n])=0$.
\end{proposition}
\begin{proof}
We begin with a standard calculation, which we illustrate formally, ignoring questions of smoothness of $\rho^u_t$ or the singularity of the logarithm. Writing the Fokker-Planck equation \eqref{eq:FKPmarginal} as $\partial_t\rho^u=\nabla_u \cdot (\rho^u(b^u+\nabla_u\log\rho^u))$, we compute
\begin{align*}
\frac{d}{dt}H_t(u) = \frac{d}{dt}\int_{\R^u} \log\frac{\rho^u_t}{\pi^u}\,\rho^u_t &= \int_{\R^u} \partial_t\rho^u_t + \int_{\R^u} \log\frac{\rho^u_t}{\pi^u} \,\partial_t\rho^u_t \\
	&= 0-\int_{\R^u} \nabla_u \log\frac{\rho^u_t}{\pi^u} \cdot (b^u_t + \nabla_u\log\rho^u_t)\, \rho^u_t \\
	&= \int_{\R^u} \bigg( (-\nabla_u\log\pi^u - b^u_t ) \cdot\nabla_u\log\frac{\rho^u_t}{\pi^u} - \Big|\nabla_u\log\frac{\rho^u_t}{\pi^u}\Big|^2\bigg)\,\rho^u_t.
\end{align*}
Recalling the definition of Fisher information from \eqref{def:H-I}, we complete the square to get
\begin{equation}
\frac{d}{dt}H_t(u) \le \frac12\E\big|\nabla_u\log\pi^u(X^u_t) + b^u_t(X^u_t)\big|^2 - \frac12I(\rho^u_t\,|\,\pi^u). \label{ineq:completesquare}
\end{equation}
To make this rigorous, Lemma 3.1 of \cite{LacLeFlem} covers our setting.
Now, recalling the formula for $b^u_t(X^u_t)$ as a conditional expectation, we add and subtract $\E[\nabla_uV(X_t)\,|\,X^u_t]$ inside the square and use the real variable inequality $(x+y)^2 \le \epsilon^{-1}x^2 + (1-\epsilon)^{-1}y^2$ to get
\begin{equation}
\begin{split}
\frac12\E\big[|\nabla_u\log\pi^u(X^u_t) + b^u_t(X^u_t)|^2\big] &\le \frac{1}{2(1-\epsilon)}\E\big|\E[\nabla_uV(X_t)-\nabla_uV(X_{t_-})\,|\,X^u_t]\big|^2 \\
	&\qquad + \frac{1}{2\epsilon}\E\big|\nabla_u\log\pi^u(X^u_t)+\E[\nabla_uV(X_t)\,|\,X^u_t]\big|^2.
\end{split} \label{pf:entropy1-AB}
\end{equation}
The first term on the right-hand side is bounded by $A^1_t(u)$ after applying Jensen to remove the conditional expectation.
The marginal density $\pi^u$ is proportional to $\int e^{-V(x)}\,dx_{[n] \setminus u}$ and in particular
\[
\nabla_u\log\pi^u(x_u) =  - \E_{Y \sim \pi}[\nabla_uV(Y)\,|\,Y^u=x_u].
\]
Hence, the second term on the right-hand side of \eqref{pf:entropy1-AB} is exactly $A^2_t(u)$. Finally, noting that the marginal $\pi^u$ obeys the log-Sobolev inequality with the same constant as $\pi$, we have $-(1/2)I(\rho^u_t\,|\,\pi^u) \le -\alpha H_t(u)$, and this completes the proof.
\end{proof}

\begin{remark} \label{re:completesquare}
In the step \eqref{ineq:completesquare}, one could instead introduce another parameter $\delta > 0$ and replace the  $-\frac12$ in front of Fisher information by $-(1-\delta)$; the factor $\frac12$ on first term in \eqref{ineq:completesquare} accordingly changes to $\frac{1}{4\delta}$. This does not seem to lead to significant benefits, just improving somewhat the exponential decay rates in Theorems \ref{th:sparse-dynamic-poly}, \ref{th:sparse-dynamic-exp}, and \ref{th:weak-dynamic}, so we have opted to simplify the presentation by taking $\delta=1/2$ instead of optimizing it. We have included the parameter $\epsilon$ in order to allow us to optimize the smallness conditions in the assumptions \eqref{asmp:expgrowth} and (B.3) as much as possible.
\end{remark}

\section{Proof in the sparse case} \label{se:proof-sparse}

The goal of this section is to prove the following dynamic generalization of Theorems \ref{th:sparse-poly} and \ref{th:sparse-exp}, which naturally require an assumption on the initial distribution $X_0 \sim \rho_0$ of the LMC iterates:
\begin{equation*}
\text{(Init)} \quad\qquad \text{There exists } C_0 \ge 0 \text{ such that } H_0(u):=H(\rho_0^u\,|\,\pi^u) \le C_0|u|, \ \forall u \subset [n]. \qquad\qquad\qquad\qquad
\end{equation*}

\begin{theorem} \label{th:sparse-dynamic-poly}
Suppose (Init) and the assumptions of Theorem \ref{th:sparse-poly} hold. Define $C$ and $h^*$ as in Theorem \ref{th:sparse-poly}. Then, for $h \le h^*$, $k \in \N$, and $u \subset [n]$,
\begin{equation*} 
H_{kh}(u) \le 2cC_0 e^{-\alpha kh/2}  \Big(\frac{2\gamma\beta^2}{\alpha} kh + 1\Big)^p  |u|   + C h |u| .
\end{equation*}
\end{theorem}

\begin{theorem} \label{th:sparse-dynamic-exp}
Suppose (Init) and the assumptions of Theorem \ref{th:sparse-exp} hold. Define $(\eta,C,h^*)$ as in Theorem \ref{th:sparse-exp}. Then, for $h \le h^*$, $k \in \N$, and $u \subset [n]$,
\[
H_{kh}(u) \le  c  C_0 e^{-\frac{\alpha\eta^2}{2(2-\eta)} k h } |u|   + C h |u|.
\] 
\end{theorem}

Theorem \ref{th:sparse-poly} (resp.\ \ref{th:sparse-exp}) follows immediately from Theorem \ref{th:sparse-dynamic-poly} (resp.\ \ref{th:sparse-dynamic-exp}), by choosing any initial condition satisfying (Init) with $C_0$ finite, such as $\rho_0=\pi$, and then sending $k\to\infty$. Indeed, lower semicontinuity of relative entropy and the Talagrand inequality \eqref{ineq:PoincareTalagrand} yield
\[
\frac{\alpha}{2}W_2^2(\pi_h^u, \pi^u) \le H(\pi_h^u\,|\,\pi^u) \le \liminf_{k\to\infty} H_{kh}(u) \le  C h |u|.
\]
The proofs of Theorems \ref{th:sparse-dynamic-poly} and \ref{th:sparse-dynamic-exp} follow most of the same path, so we will proceed at first under just the assumptions (LSI) and (A.1,2).

\subsection{The continuity term, $A^1_t$} \label{se:continuityterm-sparse}

We first estimate $A^1_t(u)$ from Proposition \ref{pr:H_t(u)firstbound}. The first key observation is that $\nabla_uV(x)$ is a $\beta$-Lipschitz function of just the coordinates in $N(u)$, so that
\begin{align*}
A^1_t(u) = \frac{1}{2(1-\epsilon)}\E\big|\nabla_u V(X_t)-\nabla_u V(X_{t_-})\big|^2 \le \frac{1}{2(1-\epsilon)}\beta^2\E\big|X^{N(u)}_t-X^{N(u)}_{t_-}\big|^2.
\end{align*}
Abbreviate $w:=N(u)$.  Using the definition of $X$,
\begin{align*}
\E \big| X^w_t - X^w_{t_-} \big|^2  &= \E\big| -(t-t_-)\nabla_w V(X_{t_-}) + \sqrt{2}(B^w_t-B^w_{t_-})\big|^2 \\
	&= (t-t_-)^2\E|\nabla_w V(X_{t_-})|^2 + 2(t-t_-)|w|.
\end{align*}
We next claim that
\begin{equation}
\E|\nabla_w V(X_{t_-})|^2 \le 2\langle\pi, |\nabla_w V|^2\rangle  + 2\beta^2W_2^2(\rho^{N(w)}_{t_-},\pi^{N(w)}). \label{pf:continuity-claim1}
\end{equation}
To see this, use again the Lipschitz assumption to get, for all $x,y \in \R^n$,
\begin{align*}
|\nabla_wV(x)|^2 &\le 2|\nabla_wV(x)-\nabla_wV(y)|^2 + 2|\nabla_wV(y)|^2 \\
	&\le 2\beta^2|x_{N(w)}-y_{N(w)}|^2 + 2|\nabla_wV(y)|^2.
\end{align*}
Kantorovich duality \cite[Theorem 1.3]{villani2021topics} states $\int f\,d\rho^{N(w)}_t - \int g\,d\pi^{N(w)} \le 	W_2^2(\rho^{N(w)}_t,\pi^{N(w)})$ for any continuous functions $f$ and $g$ on $\R^{N(w)}$ satisfying $f(x_{N(w)})-g(y_{N(w)})\le|x_{N(w)}-y_{N(w)}|^2$ pointwise. 
From this we deduce \eqref{pf:continuity-claim1}. 

Next, using $\nabla V = -\nabla\log \pi$, integration by parts yields 
\begin{equation*}
\langle\pi, |\nabla_w V|^2\rangle = -\int_{\R^n} \nabla_w V \cdot \nabla_w \pi = \int_{\R^n} \Delta_w V\,\pi \le \beta |w|, 
\end{equation*}
where we used $\partial_{ii}V \le \|\nabla^2V\|_{\mathrm{op}} \le \beta$ for each $i \in w$.
Putting it together,
\begin{equation*}
A^1_t(u) \le \frac{\beta^2}{2(1-\epsilon)}\Big( 2\beta (t-t_-)^2|w| + 2\beta^2 (t-t_-)^2W_2^2(\rho^{N(w)}_{t_-},\pi^{N(w)})+ 2(t-t_-)|w|\Big).
\end{equation*}
Use the Talagrand inequality \eqref{ineq:PoincareTalagrand}, and plug in $w=N(u)$ and $N(w)=N_2(u)$ to get
\begin{equation}
A^1_t(u) \le \frac{\beta^2h}{1-\epsilon} (\beta h + 1)|N(u)| + \frac{2\beta^4h^2}{\alpha(1-\epsilon)}  H_{t_-}(N_2(u)). \label{pf:continuityterm1}
\end{equation}

\subsection{The hierarchy term, $A^2_t$} \label{se:hierarchyterm-sparse}

We next estimate $A^2_t(u)$ from Proposition \ref{pr:H_t(u)firstbound}.
We begin with the following simple lemma, which explains how the metric $W_1$ controls vector-valued and not just scalar-valued Lipschitz test functions. Here and in the sequel we find it convenient to use the bracket notation $\langle \mu,f\rangle = \int f\,d\mu$ for integration.

\begin{lemma} \label{le:T1multivariate}
Let $m,k \in \N$.
Suppose a probability measure $\mu$ on $\R^m$ satisfies
\[
W_1^2(\mu,\nu) \le \gamma H(\nu\,|\,\mu), \qquad \forall \nu \in \P(\R^m).
\]
Then, for any $L$-Lipschitz function $f : \R^m \to \R^k$, we have
\[
\big| \langle \mu-\nu,f\rangle\big|^2 \le \gamma L^2 H(\nu\,|\,\mu), \qquad \forall \nu \in \P(\R^m).
\]
\end{lemma}
\begin{proof}
The case $k=1$ is an immediate consequence of Kantorovich duality \cite[Theorem 1.14]{villani2021topics}. In general, let $u \in \R^k$ be a unit vector and apply the $k=1$ case to the scalar function $u\cdot f$ to get
\[
\big( \langle \mu-\nu,f\rangle \cdot u\big)^2 \le \gamma L^2 H(\nu\,|\,\mu), \qquad \forall \nu \in \P(\R^m).
\]
Take the supremum over unit vectors on the left-hand side to complete the proof.
\end{proof}

Let us write $\rho^{w|u}_{t,x_u}$ for the conditional law of $X^w_t$ given $X^u_t=x_u$, for each $t > 0$ and $u,w \subset [n]$. Define $\pi^{w|u}_{x_u}$ similarly.
Recall by the sparsity assumption that $\nabla_uV(x)$ depends only on the coordinates $x_{N(u)}$. For $x_u \in \R^u$ let us abuse notation slightly by writing $\nabla_uV(x_u,\cdot)$ for the function of $x_{N(u)\setminus u}$ obtained by fixing the coordinates indexed by $u$.
Then $A^2_t(u)$ can be written as
\begin{equation*}
A^2_t(u) = \frac{1}{2\epsilon}\int_{\R^u} \big|\langle \rho^{N(u) \setminus u|u}_{t,x_u}-\pi^{N(u) \setminus u|u}_{x_u}, \,\nabla_u V(x_u,\cdot)\rangle\big|^2  \,\rho^u_t(dx_u).
\end{equation*}
The assumed transport inequality (A.2) along with Lemma \ref{le:T1multivariate} imply
\begin{equation*}
A^2_t(u) \le \frac{\gamma\beta^2}{\alpha\epsilon}\int_{\R^u}H(\rho^{N(u) \setminus u|u}_{t,x_u}\,|\,\pi^{N(u) \setminus u|u}_{x_u})  \,\rho^u_t(dx_u).
\end{equation*}
By the chain rule for relative entropy, the integral is exactly $H_t(N(u))-H_t(u)$, and so
\begin{equation}
A^2_t(u) \le \frac{\gamma\beta^2}{\alpha\epsilon}\big(H_t(N(u))-H_t(u)\big).  \label{pf:hierarchyterm1}
\end{equation}
Note that if $u$ has no neighbors (for instance, if $u=[n]$) then $N(u)=u$ and $A^2_t(u)=0$.

\begin{remark}
The above paragraph is the only place that assumption (A.2) is used. It could thus be weakened, by instead directly assuming
\[
 \big|\langle \mu - \pi^{N(u) \setminus u|u}_{x_u}, \,\nabla_u V(x_u,\cdot)\rangle\big|^2 \le \frac{2\gamma\beta^2}{\alpha}H(\mu\,|\,\pi^{N(u) \setminus u|u}_{x_u}), \qquad \forall u \subset [n], \ x_u \in \R^u, \ \mu \in \P(\R^u).
\]
\end{remark}

\subsection{Developing the hierarchy} \label{se:developing-sparse}

Combining Proposition \ref{pr:H_t(u)firstbound} with the estimates \eqref{pf:continuityterm1} and \eqref{pf:hierarchyterm1}, we get
\begin{equation}
\begin{split}
\frac{d}{dt}H_t(u) \le & \frac{\beta^2h}{1-\epsilon}(\beta h + 1)|N(u)| + \frac{2\beta^4h^2 }{\alpha(1-\epsilon)} H_{t_-}(N_2(u))  \\
	&+ \frac{\gamma\beta^2}{\alpha\epsilon}\big(H_t(N(u))-H_t(u)\big) - \alpha H_t(u).
\end{split} \label{ineq:sparsehierarchy}
\end{equation}
This holds for all $u \subset [n]$. To understand this system of inequalities, we follow the idea of \cite[Section 1.4]{LackerYeungZhou}, introducing an auxiliary stochastic process for which \eqref{ineq:sparsehierarchy} becomes the associated Feynman-Kac formula (more precisely, an inequality version). This auxiliary process is a continuous-time Markov chain, taking values in the space of subsets of $[n]$, which serves as a convenient tool for analyzing \eqref{ineq:sparsehierarchy}. View $H_t$ for each $t$ as a vector indexed by sets, i.e., $H_t=(H_t(u))_{u \subset [n]} \in \R^{2^{[n]}}$. The right-hand side of \eqref{ineq:sparsehierarchy} is an affine function of $H_t$. In particular, define
\[
C_t(u) := \frac{\beta^2h}{1-\epsilon}(\beta h + 1)|N(u)| +  \frac{2\beta^4 h^2}{\alpha(1-\epsilon)} H_{t_-}(N_2(u)),
\]
and define a linear operator $\mathsf{A} : \R^{2^{[n]}} \to \R^{2^{[n]}}$ by
\begin{equation}
\mathsf{A} F(u) :=  \frac{\gamma\beta^2}{\alpha\epsilon}\big(F(N(u))-F(u)\big), \qquad \forall F \in \R^{2^{[n]}}, \ u \subset [n]. \label{def:Aoperator-sparse}
\end{equation}
This operator has the structure of an infinitesimal generator; that is, the corresponding (rate) matrix has no negative entries outside the main diagonal, and each row sum is zero. The associated continuous-time Markov process $(\X_t)_{t \ge 0}$ takes values in $2^{[n]}$ and transitions from $u$ to $N(u)$ at rate $\gamma\beta^2/\alpha\epsilon$, for each $u$, with no other transitions. 
The semigroup generated by $\mathsf{A}$ is then related to $\X$ via $e^{t\mathsf{A}}F(u)=\E[F(\X_t)\,|\,\X_0=u]$ for any $F \in \R^{2^{[n]}}$ and $t \ge 0$. For later use, we note that that behavior of the process $\X$ is extremely simple: When initialized from $\X_0=u$ it simply follows the path $u \to N_1(u) \to N_2(u) \to \cdots$, with Exp($\lambda$) holding times between jumps, where $\lambda:=\gamma\beta^2/\alpha\epsilon$. In particular, we can construct $\X$ as $\X_t=N_{\Lambda(t)}(u)$, where $(\Lambda(t))_{t \ge 0}$ denotes a Poisson process with rate $\lambda$, started from $\Lambda(0)=0$. 

Now, the inequality \eqref{ineq:sparsehierarchy} can be written in vector form as
\[
\frac{d}{dt}H_t \le C_t + (\mathsf{A}-\alpha \mathsf{I})H_t,
\]
where $\mathsf{I}$ denotes the identity operator, and where we understand the inequality to be coordinatewise. 
For $t \ge 0$ the operator $e^{t\mathsf{A}}$ is monotone with respect to coordinatewise order, in the sense that $e^{t\mathsf{A}}F \le e^{t\mathsf{A}}G$ if $F \le G$. From this  it follows that, for any $t > s > 0$,
\begin{equation}
\frac{d}{ds}e^{-\alpha (t-s)}e^{(t-s)\mathsf{A}}H_s = e^{-\alpha (t-s)}e^{(t-s)\mathsf{A}}\bigg(\frac{d}{ds}H_s - (\mathsf{A} -\alpha \mathsf{I}) H_s\bigg)  \le e^{-\alpha (t-s)}e^{(t-s)\mathsf{A}}C_s. \label{pf:Gronwall}
\end{equation}
Let $k \in \N$ and integrate from $s=(k-1)h$ to $t=kh$ to get 
\begin{equation*}
H_{kh} \le e^{-\alpha h}e^{h\mathsf{A}}H_{(k-1)h} + \int_{(k-1)h}^{kh} e^{-\alpha (kh-s)}e^{(kh-s)\mathsf{A}}C_s\,ds.
\end{equation*}
For any function $F \in \R^{2^{[n]}}$ satisfying $F(u) \le F(N(u))$ for all $u \subset [n]$, such as $F=C_s$, it follows easily from the probabilistic representation of the semigroup that $t \mapsto e^{t\mathsf{A}}F(u)$ is nondecreasing. Noting also that $C_s=C_{(k-1)h}$ for $s \in [(k-1)h,kh)$, we get
\begin{equation}
H_{kh} \le e^{-\alpha h}e^{h\mathsf{A}}H_{(k-1)h} + \frac{1}{\alpha}(1-e^{-\alpha h})e^{h\mathsf{A}}C_{(k-1)h}. \label{pf:Hkh-bound1}
\end{equation}
Next, define a linear operator $\mathsf{N} : \R^{2^{[n]}} \to \R^{2^{[n]}}$ and a function $G \in \R^{2^{[n]}}$,
\begin{equation}
\mathsf{N}F(u) := F(N(u)),\qquad G(u) := \frac{\beta^2h}{1-\epsilon}(\beta h + 1)|u|. \label{pf:Gdef}
\end{equation}
This way, we may rewrite
\[
C_{(k-1)h} = \mathsf{N}G +  \frac{2\beta^4 h^2}{\alpha(1-\epsilon)} \mathsf{N}^2 H_{(k-1)h}.
\]
Plugging this into \eqref{pf:Hkh-bound1} and combining the $H_{(k-1)h}$ terms, 
\begin{equation*}
H_{kh} \le \Big(e^{-\alpha h}e^{h\mathsf{A}} + \frac{2\beta^4h^2}{\alpha^2 (1-\epsilon)}(1-e^{-\alpha h}) e^{h\mathsf{A}}\mathsf{N}^2 \Big)H_{(k-1)h} + \frac{1-e^{-\alpha h}}{\alpha} e^{h\mathsf{A}}\mathsf{N} G.
\end{equation*}
The operators $\mathsf{N}$ and $\mathsf{A}$ commute, because for all $F \in \R^{2^{[n]}}$ we have
\begin{equation*}
\mathsf{A}\mathsf{N}F(u) = \frac{\gamma\beta^2}{\alpha\epsilon}\big(F(N(N(u))) - F(N(u))\big) = \mathsf{N}\mathsf{A}F(u).
\end{equation*}
Hence,  
\begin{equation*}
H_{kh} \le \Big(e^{-\alpha h}\,\mathsf{I} + \frac{2\beta^4h^2}{\alpha^2 (1-\epsilon)}(1-e^{-\alpha h})  \mathsf{N}^2 \Big) e^{h\mathsf{A}}H_{(k-1)h} + \frac{1-e^{-\alpha h}}{\alpha} e^{h\mathsf{A}}\mathsf{N} G.
\end{equation*}
Iterate this inequality to find
\begin{equation}
\begin{split}
H_{kh} \le \ &\Big(e^{-\alpha h}\,\mathsf{I} + \frac{2\beta^4h^2}{\alpha^2 (1-\epsilon)}(1-e^{-\alpha h})  \mathsf{N}^2 \Big)^k e^{kh\mathsf{A}} H_0 \\
	&+ \frac{1-e^{-\alpha h}}{\alpha}\sum_{j=0}^{k-1} \Big(e^{-\alpha h}\,\mathsf{I} + \frac{2\beta^4h^2}{\alpha^2 (1-\epsilon)}(1-e^{-\alpha h})  \mathsf{N}^2 \Big)^{j} e^{(j+1)h\mathsf{A}} \mathsf{N} G.
\end{split} \label{pf:operator-iteration}
\end{equation}
Here we used the fact that the operators $e^{s\mathsf{A}}$ and $\mathsf{N}$ are monotone with respect to pointwise order.
The rest of the analysis involves specific estimates of how the operators $\mathsf{N}$ and $e^{t\mathsf{A}}$ act on the ``size function'' $S(u):=|u|$, and we will now treat the two Theorems \ref{th:sparse-dynamic-poly} and \ref{th:sparse-dynamic-exp} separately.

\subsection{The polynomial growth case, Theorem \ref{th:sparse-dynamic-poly}}

We now impose the polynomial growth assumption \eqref{asmp:polygrowth}, and we take $\epsilon=1/2$ for simplicity.
By a union bound,
\begin{equation}
|N_{k+1}(u)| \le \sum_{i \in u} |N_{k+1}(i)| \le c|u|(1+k^p), \quad \forall u \subset [n], \ k \ge 0. \label{pf:combinatorialbound-poly}
\end{equation}
In particular, $|N(u)| \le c|u|$, or $\mathsf{N}S \le cS$ coordinatewise.
Next, we exploit the stochastic representation of the semigroup, writing $\X_t=N_{\Lambda(t)}(\X_0)$ as discussed above. This yields
\begin{equation*}
e^{t\mathsf{A}}S(u) = \E|N_{\Lambda(t)}(u)| \le c|u|(1+\E \Lambda^p(t))  \le c f(t)S(u), 
\end{equation*}
where $f(t) := 1 + (\lambda t + p)^p$, and the last step uses a Poisson moment bound due to \cite[Corollary 1]{ahle2022sharp}.
This implies $\mathsf{N}^2 e^{(j+1)h\mathsf{A}}S \le  c^2 f((j+1)h)S$, and so
\begin{equation}
\Big(e^{-\alpha h}\,\mathsf{I} + \frac{4\beta^4h^2}{\alpha^2  }(1-e^{-\alpha h})  \mathsf{N}^2 \Big)^{j}  e^{(j+1)h\mathsf{A}}S  \le  c b^j  f((j+1)h) S, \label{pf:bigopbound-cor}
\end{equation}
where the constant $b$ is defined by
\[
 b := e^{-\alpha h} + \frac{4c^2 \beta^4h^2}{\alpha^2 }(1-e^{-\alpha h}) .
\]
Since $\alpha \le \beta$ and $c \ge 1$, we have $\alpha h \le \alpha^2/4c\beta^2 \le 1/4$ and thus $1-e^{-\alpha h}\ge e^{-1/4}\alpha h$ and 
\[
 b \le e^{-\alpha h} + \frac14(1-e^{-\alpha h}) = 1 - \frac34(1-e^{-\alpha h}) \le 1 - \frac34 e^{-1/4}\alpha h \le 1-\frac12 \alpha h \le e^{-\alpha h/2}.
\]
Plugging this into \eqref{pf:bigopbound-cor}, and using again $\mathsf{N}S \le cS$, we find
\begin{align*}
\Big(e^{-\alpha h}\,\mathsf{I} + \frac{4\beta^4h^2}{\alpha^2  }(1-e^{-\alpha h})  \mathsf{N}^2 \Big)^{k}  e^{kh\mathsf{A}}S  &\le  c e^{-\alpha kh/2} f(kh)S, \\
\Big(e^{-\alpha h}\,\mathsf{I} + \frac{4\beta^4h^2}{\alpha^2  }(1-e^{-\alpha h})  \mathsf{N}^2 \Big)^{j}  e^{(j+1)h\mathsf{A}} \mathsf{N}S  &\le  c^2 e^{-\alpha j h/2} f((j+1)h)S.
\end{align*}
From the definition of $G$ in \eqref{pf:Gdef} with $\epsilon=1/2$, and from the inequality $\beta h \le 1/4$, we get $G  \le \frac52\beta^2hS$.
Thus, using \eqref{pf:operator-iteration} and the assumption (Init) in the form $H_0 \le C_0S$, we have
\begin{equation} 
H_{kh}(u) \le C_0c e^{-\alpha kh/2} f(kh) |u|   + \frac52 c^2\beta^2h  \frac{1-e^{-\alpha h}}{\alpha} |u| \sum_{j=0}^\infty e^{-\alpha j h/2} f((j+1)h). \label{pf:poly2}
\end{equation}
We simplify the sum by plugging in $f(t) \le 2(\lambda t + p)^p$ (which holds because $p \ge 1$).
We then use the simple calculation 
\[
\sup_{x \ge 0} (\lambda x + p)  e^{-\alpha  x/4p } \le \frac{4p\lambda}{\alpha} \vee p \le \frac{4p\lambda}{\alpha} + p
\]
to get
\begin{align*}
\sum_{j=0}^\infty e^{-\alpha j h/2} f((j+1)h) &\le 2e^{\alpha h/2}\sum_{j=1}^\infty e^{-\alpha jh/2}(\lambda jh + p)^p \\
	&\le  2e^{\alpha h/2}p^p\Big(\frac{4\lambda}{\alpha} + 1\Big)^p  \sum_{j=1}^\infty e^{-\alpha jh/4} \\
	&= \frac{2e^{\alpha h/2}p^p }{e^{\alpha h/4}-1} \Big(\frac{4\lambda}{\alpha} + 1\Big)^p.
\end{align*} 
Plug this into \eqref{pf:poly2}, and use $e^{\alpha h/2} \le e^{1/8} \le 2$ and $(1-e^{-\alpha h})/(e^{\alpha h/4}-1) \le 4$:
\begin{equation*} 
H_{kh}(u) \le C_0c e^{-\alpha kh/2} f(kh) |u|   + 40c^2 p^p\frac{\beta^2}{\alpha}\Big(\frac{4\lambda}{\alpha} + 1\Big)^ph|u| .
\end{equation*}
Finally, plug in $\lambda = 2\gamma\beta^2/\alpha $ to complete the proof of Theorem \ref{th:sparse-dynamic-poly}. \hfill\qedsymbol

\subsection{The exponential growth case, Theorem \ref{th:sparse-dynamic-exp}}

We next prove Theorem \ref{th:sparse-dynamic-exp}, starting from \eqref{pf:operator-iteration}. We will first show that the size function $S(u)=|u|$ satisfies
\begin{equation}
e^{t\mathsf{A}} S \le ce^{\lambda t(r-1)} S , \quad \forall t > 0 . \label{pf:Gselfbound}
\end{equation}
We first use the growth assumption along with a union bound to get $|N_{k+1}(u)| \le c|u|r^k$ for all $k \ge 0$ and $u \subset [n]$. In particular, $|N(u)| \le c|u|$, or $\mathsf{N}S \le cS$.
Using the stochastic representation of the semigroup, we deduce
\begin{equation*}
e^{t\mathrm{A}}S(u) =  \E |N_{\Lambda(t)}(u)| \le c|u|\E \, r^{\Lambda(t)} = c|u|e^{\lambda t(r-1)}= c e^{\lambda t(r-1)}S(u).
\end{equation*}
This proves \eqref{pf:Gselfbound}, which we then combine with $\mathsf{N}S \le cS$ to get $\mathsf{N} e^{(j+1)h\mathsf{A}}S \le  c^2 e^{\lambda (j+1)h(r-1)}S$, as well as 
\begin{equation}
\Big(e^{-\alpha h}\,\mathsf{I} + \frac{2\beta^4h^2}{\alpha^2 (1-\epsilon)}(1-e^{-\alpha h})  \mathsf{N}^2 \Big)^{j}  e^{(j+1)h\mathsf{A}}S  \le  c b^j  e^{h\lambda (j+1)(r-1) } S, \label{pf:bigopbound}
\end{equation}
where we define the constant
\[
 b := e^{-\alpha h} + \frac{2c^2\beta^4h^2}{\alpha^2 (1-\epsilon)}(1-e^{-\alpha h}) .
\]
We estimate $b$ in simpler terms as follows, with the key point being that $h$ can be taken small enough to make $b < 1$. We start from the elementary inequality
\begin{equation}
e^{-x}+\bigg(1-\frac{2\epsilon(1-\epsilon)}{ 1-e^{-2(1-\epsilon)}}\bigg)(1-e^{-x}) \le 1 - \epsilon x, \quad \text{for } 0 \le x \le 2(1-\epsilon). \label{ineq:1-eps.x}
\end{equation}
Indeed, this is equivalent to the inequality
\begin{equation*}
\frac{2(1-\epsilon)}{ 1-e^{-2(1-\epsilon)}} \ge \frac{x}{ 1-e^{-x}}, \quad \text{for } 0 \le x \le 2(1-\epsilon),
\end{equation*}
whose right-hand side is an increasing function of $x$.
Now, suppose
\begin{equation}
\widehat{h} := \min\bigg(\frac{2(1-\epsilon)}{\alpha}, \ \frac{\alpha}{c\beta^2\sqrt{2}}\bigg((1-\epsilon)\Big(1 - \frac{2 \epsilon(1-\epsilon)}{1-e^{-2(1-\epsilon)}}\Big)\bigg)^{1/2}\bigg). \label{def:h-hat}
\end{equation}
Then, for $h \le \widehat{h}$, the first term in the min ensures that $\alpha h \le 2(1-\epsilon)$, and the second term lets us upper bound $\frac{2c^2 \beta^4h^2}{\alpha^2 (1-\epsilon)}  \le 1-\frac{2\epsilon(1-\epsilon)}{1-e^{-2(1-\epsilon)}}$ in the definition of $b$. Applying \eqref{ineq:1-eps.x} then yields
\[
b \le  1-\epsilon\alpha h \le e^{ - \epsilon \alpha h}.
\]
Plugging this into \eqref{pf:bigopbound}, and using again $\mathsf{N}S \le cS$, we find
\begin{align*}
\Big(e^{-\alpha h}\,\mathsf{I} + \frac{2\beta^4h^2}{\alpha^2 (1-\epsilon)}(1-e^{-\alpha h})  \mathsf{N}^2 \Big)^{k}  e^{kh\mathsf{A}}S  &\le  c e^{\lambda (r-1)kh - \epsilon \alpha k h }  S, \\
\Big(e^{-\alpha h}\,\mathsf{I} + \frac{2\beta^4h^2}{\alpha^2 (1-\epsilon)}(1-e^{-\alpha h})  \mathsf{N}^2 \Big)^{j}  e^{(j+1)h\mathsf{A}} \mathsf{N}S  &\le  c^2 e^{\lambda (r-1)(j+1)h - \epsilon \alpha j h }  S.
\end{align*}
Finally, recall from \eqref{pf:Gdef} that $G(u)$ is just a constant multiple of $S(u)=|u|$ by definition, and also $H_0 \le C_0S$ by assumption. Hence, returning to \eqref{pf:operator-iteration}, we find
\begin{align*} 
H_{kh}(u) &\le C_0  c e^{\lambda (r-1)kh - \epsilon \alpha k h } |u|   + \frac{ c^2 \beta^2h}{1-\epsilon}(\beta h + 1)\frac{1-e^{-\alpha h}}{\alpha}\sum_{j=0}^{k-1}e^{\lambda (r-1)(j+1)h - \epsilon \alpha j h } |u|.
\end{align*}
Because of the assumption $r -1 < \alpha^2/\gamma\beta^2$, we may choose $\epsilon$ close to 1 so that
\begin{align*}
\tau := \epsilon \alpha - \lambda(r-1) &= \epsilon \alpha - \frac{\gamma\beta^2}{\alpha\epsilon}(r-1) > 0.
\end{align*}
Then,
\begin{align*}
\frac{1-e^{-\alpha h}}{\alpha}\sum_{j=0}^{k-1}e^{\lambda (r-1)(j+1)h - \epsilon \alpha j h } &\le \frac{1-e^{-\alpha h}}{\alpha(1-e^{-\tau h})}e^{\lambda h(r-1)} \le \Big(\frac{1}{\tau} + h\Big)e^{\lambda h(r-1)}.
\end{align*}
where the last line used $1-e^{-\alpha h} \le \alpha h$ and $\tau h /(1-e^{-\tau h}) \le 1+\tau h$.
Thus,
\begin{align*} 
H_{kh}(u) &\le C_0  c e^{-\tau k h } |u|   + \frac{ c^2\beta^2h}{1-\epsilon}(\beta h + 1) \Big(\frac{1}{\tau} + h\Big)e^{\lambda h(r-1)} |u|.
\end{align*}
This is valid for $h \le \widehat{h}$, where $\widehat{h}$ was defined in \eqref{def:h-hat}. We have thus proven Theorem \ref{th:sparse-dynamic-exp}, except with $\widehat{h}$ instead of $h^*$, and with $C$ replaced by
\[
\widehat{C}_h =\frac{ c^2r\beta^2 }{1-\epsilon}(\beta h + 1) \Big(\frac{1}{\tau} + h\Big)e^{h(r-1)\frac{\gamma\beta^2}{\alpha\epsilon}}.
\]

The rest of the proof is just a somewhat tedious simplification of the constants, to show that $\widehat{h} \ge h^*$ and that $\widehat{C}_h \le C$ for $h \le h^*$, where $h^*$ and $C$ were defined in Theorem \ref{th:sparse-exp}. Take $\epsilon= \frac12 + \frac{\gamma\beta^2}{2\alpha^2}(r-1)$, which is the midpoint between $\frac{\gamma\beta^2}{\alpha^2}(r-1)$ and 1. Then the exponent $\tau$ becomes
\begin{align*}
\tau &= \frac{\alpha}{2} + \frac{\gamma\beta^2}{2\alpha}(r-1) - \frac{\gamma\beta^2}{\frac{\alpha}{2} + \frac{\gamma\beta^2}{2\alpha}(r-1)}(r-1) \\
	&= \frac{\alpha}{2} \frac{\Big(1 - \frac{\gamma\beta^2}{\alpha^2}(r-1)\Big)^2 }{1 + \frac{\gamma\beta^2}{\alpha^2}(r-1)}.
\end{align*}
Recalling the definition $\eta = 1 - \frac{\gamma\beta^2}{\alpha^2}(r-1)$ from the theorem statement, we get 
\begin{align*}
 \tau = \frac{\alpha\eta^2}{2(2-\eta)}, \quad 1-\epsilon  &= \frac{\eta}{2}, \quad 2\epsilon(1-\epsilon) = \frac12\eta(2-\eta).
\end{align*} 
The constants then become
\begin{align*}
\widehat{C}_h &= \frac{2 c^2\beta^2}{\eta}  (\beta h + 1) \bigg( \frac{4-2\eta}{\alpha\eta^2 } + h\bigg)e^{h(r-1)\frac{2\gamma\beta^2}{\alpha(2-\eta)}}, \\ 
\widehat{h} &= \min\bigg\{\frac{\eta}{\alpha}, \frac{\alpha\sqrt{\eta}}{2\beta^2c } \bigg(1-\frac{\frac12\eta(2-\eta)}{1-e^{-\eta} }\bigg)^{1/2}\bigg\} .
\end{align*}
To bound $\widehat{h} \ge h^*$, we use the inequality
\begin{equation}
\bigg(1-\frac{\frac12\eta(2-\eta)}{1-e^{-\eta} }\bigg)^{1/2} \ge \frac{\eta}{\sqrt{6}}. \label{ineq:etabound-h}
\end{equation}
To see this, note that $f(x):=1-e^{-x}-\frac12 x(2-x)$ satisfies $f(0)=f'(0)=f''(0)=0$, with $f''(x)=1-e^{-x}$ concave. This concavity implies $1-e^{-t}\ge (t/\eta)(1-e^{-\eta})$ for $0 \le t \le \eta$, so that
\[
1-e^{-\eta}-\frac12 \eta(2-\eta) = \int_0^\eta\int_0^x (1-e^{-t})\,dtdx  \ge \frac{\eta^2}{6}(1-e^{-\eta}).
\] 
Thus, using also $2\sqrt{6} \le 5$, we have shown that
\[
\widehat{h} \ge \min\bigg\{\frac{\eta}{\alpha}, \frac{\alpha \eta^{3/2}}{5\beta^2c  }  \bigg\} \ge  \frac{\alpha \eta^{3/2}}{5\beta^2c  } = h^*,
\]
where the last step used $0 < \eta < 1 \le c$ and $\alpha \le \beta$ (which was noted after \eqref{ineq:PoincareTalagrand}).
We finally bound $\widehat{C}_h \le C$ for $h \le h^*$. Note that the second case of the minimum in $h^*$ and the facts that $\eta \le 1 \le cr$ and $\alpha \le \beta$ imply $h^* \le 1/5\beta$. Hence, $\beta h +1 \le 6/5$ for $h \le h^*$. 
Using also $h^* \le \eta/\alpha$ and   $4-2\eta+\eta^3\le 4$, we may bound $\widehat{C}_h$ for $h \le \widehat{h}$ by
\begin{align*}
\widehat{C}_h &\le \frac{2c^2\beta^2}{\eta} \cdot \frac65\cdot \frac{4-2\eta+\eta^3}{\alpha\eta^2 }e^{h(r-1)\frac{2\gamma\beta^2}{\alpha }}  \le \frac{10\beta^2c^2}{\alpha\eta^3 } e^{\gamma(r-1)/c}   = C,
\end{align*} 
where we bounded the exponent using $h \le \alpha  /2\beta^2c$. This completes the proof of Theorem \ref{th:sparse-dynamic-exp}. \hfill \qedsymbol

\section{Proof in the weak interaction case} \label{se:proof-weak}

We next turn to the weak interaction case.
The following dynamic theorem implies Theorem \ref{th:weak}, by the same logic detailed after the statement of Theorem \ref{th:sparse-dynamic-exp}.

\begin{theorem} \label{th:weak-dynamic}
Grant assumptions (LSI), (B.1--3), and (Init). Define $(\eta,C,h^*)$ be as in Theorem \ref{th:weak}. Then, for $h \le h^*$, $k \in \N$, and $u \subset [n]$,
\[
H_{kh}(u) \le C_0 e^{-\frac{\alpha\eta^2}{2(2-\eta)} k h } |u|   + C h |u|.
\] 
\end{theorem}

The proof will follow similar steps but is more involved due to the more complex structure of the potential \eqref{def:Vstructure}.
We will need some simple combinatorial identities. For $u \subset [n]$ and $F \in \R^{2^{[n]}}$,
\begin{equation}
\sum_{i \in u}\sum_{w \ni i}F(w) = \sum_{w \cap u \neq\emptyset} F(w)|u \cap w|, \label{eq:combinatorial1}
\end{equation}
where the second summation is over those $w \subset [n]$ which contain $i$, and the third is over $w \subset [n]$ which intersect $u$.
Indeed, in the summation on the left-hand side, each set $w$ is counted $|u \cap w|$ times.
More generally, if $G(w)=(G_i(w))_{i \in w}$ is in $\R^w$ for each $w \subset [n]$,
\begin{equation*}
\sum_{i \in u}\sum_{w \ni i}G^2_i(w) = \sum_{w \cap u \neq\emptyset}\sum_{i \in u \cap w} G^2_i(w) \le \sum_{w \cap u \neq\emptyset} |G(w)|^2.
\end{equation*}
We will use this in combination with Cauchy-Schwarz, in the form 
\begin{equation}
\sum_{i \in u} \bigg(\sum_{w \ni i} G_i(w)\bigg)^2 \le \sum_{i \in u} \bigg(\sum_{w \ni i} L_w \bigg)\bigg(\sum_{w \ni i} \frac{G_i^2(w)}{L_w }\bigg) \le M_0 \sum_{w \cap u \neq\emptyset} \frac{|G(w)|^2}{L_w},\label{ineq:CScombinatorial}
\end{equation}
where $M_0$ was defined in \eqref{def:MpRp-constants}. To avoid division by zero in \eqref{ineq:CScombinatorial}, we will be careful only to use it in the case that $G(w)=0$ for every $w \subset [n]$ such that $L_w=0$, and in this case we adopt the convention that $|G(w)|^2/L_w=0$.

Let us note for later use that \eqref{ineq:CScombinatorial} implies the Lipschitz constant of $\nabla V$ is bounded by $M_0$. Indeed,  use the specific form of the potential $V$ to write 
\[
\partial_i V(x) = \sum_{ w \ni i}\partial_i V_w(x_w), \qquad i \in [n].
\]
Appling \eqref{ineq:CScombinatorial} with $u=[n]$ then yields
\begin{align}
|\nabla V(x)-\nabla V(y)|^2 &= \sum_{i\in [n]}\bigg(\sum_{w \ni i}(\partial_iV_w(x_w) - \partial_iV_w(y_w))\bigg)^2 \nonumber \\
	&\le M_0 \sum_{ w \subset [n]} \frac{|\nabla_wV_w(x_w) - \nabla_wV_w(y_w)|^2}{L_w} \nonumber \\
	&\le M_0 \sum_{ w \subset [n]} L_w|x_w-y_w|^2 \nonumber \\
	&\le M_0^2|x-y|^2. \label{ineq:Lipschitz-weak}
\end{align}

\subsection{The continuity term, $A^1_t$}

Using \eqref{ineq:CScombinatorial} and the fact that $\nabla_wV_w$ is $L_w$-Lipschitz by (B.1), we get
\begin{align}
A^1_t(u) &= \frac{1}{2(1-\epsilon)} \sum_{i \in u}\E\bigg(\sum_{ w \ni i}\big(\partial_iV_w(X^w_t) - \partial_i V_w(X^w_{t_-})\big)\bigg)^2 \nonumber \\
	&\le \frac{M_0}{2(1-\epsilon)}\sum_{w\cap u \neq \emptyset} \frac{1}{L_w}  \E\big|\nabla_w V_w(X^w_t) - \nabla_wV_w(X^w_{t_-})\big|^2 \nonumber \\
	&\le  \frac{M_0}{2(1-\epsilon)}\sum_{w\cap u \neq \emptyset} L_w  \E\big| X^w_t - X^w_{t_-} \big|^2 \nonumber \\
	&=  \frac{M_0}{2(1-\epsilon)}\sum_{w\cap u \neq \emptyset} L_w \Big( (t-t_-)^2\E|\nabla_w V(X_{t_-})|^2 + 2(t-t_-)|w|\Big). \label{pf:weak-continuity1}
\end{align}
We would like to control the expectation in terms of its analog under $\pi$, but this is trickier to achieve than in Section \ref{se:continuityterm-sparse}. The issue is that $\nabla_wV$ depends on all coordinates, but we do not want a distance between the full distributions $\rho_t$ and $\pi$ to appear, as this would lose the delocalization effect. Instead, we  note that $\nabla V$ has mean zero under $\pi$ to write
\begin{align}
\E|\nabla_w V(X_{t_-})|^2 &= \sum_{i \in w}\E\bigg(\sum_{v \ni i}\big(\partial_i V_v(X_{t_-})-\langle\pi^v,\partial_iV_v\rangle)\bigg)^2 \nonumber \\
	&\le M_0 \sum_{v \cap w \neq \emptyset} \frac{1}{L_v} \E|\nabla_v V_v(X_{t_-}) - \langle\pi^v,\nabla_vV_v\rangle|^2. \label{pf:weak-continuity2}
\end{align}
Now, each term in the sum can be estimated using Kantorovich duality as in \eqref{pf:continuity-claim1}:
\begin{equation*}
\E|\nabla_v V_v(X_{t_-})- \langle\pi^v,\nabla_vV_v\rangle|^2 \le 2\int_{\R^v}|\nabla_v V_v- \langle\pi^v,\nabla_vV_v\rangle|^2\,d\pi^v + 2L_v^2W_2^2(\rho^v_{t_-},\pi^v).
\end{equation*} 
Using the Poincar\'e inequality \eqref{ineq:PoincareTalagrand} and the fact that $\nabla_vV_v$ is $L_v$-Lipschitz,
\begin{equation*}
\int_{\R^v}|\nabla_v V_v- \langle\pi^v,\nabla_vV_v\rangle|^2\,d\pi^v = \sum_{i \in v}\big(\langle \pi^v, (\partial_iV_v)^2\rangle - \langle \pi^v, \partial_iV_v\rangle^2\big) \le \frac{L_v^2}{\alpha}|v| .
\end{equation*} 
The Talagrand inequality \eqref{ineq:PoincareTalagrand} yields $W_2^2(\rho^v_{t_-},\pi^v) \le (2/\alpha)H_{t_-}(v)$, and we return to \eqref{pf:weak-continuity2} to get
\begin{equation}
\E|\nabla_w V(X_{t_-})|^2 \le \frac{2M_0}{\alpha} \sum_{v \cap w \neq \emptyset} L_v \big(2H(\rho^v_{t_-}\,|\,\pi^v) + |v|\big). \label{pf:weak-continuity3}
\end{equation}
To streamline notation, we define an operator $\mathsf{N} : \R^{2^{[n]}}\to \R^{2^{[n]}}$ by
\begin{equation}
\mathsf{N}F(u) := \sum_{w \cap u \neq \emptyset}L_wF(w). \label{pf:Ndef-weak}
\end{equation}
Using also the ``size function'' $S(u)=|u|$, we may thus write \eqref{pf:weak-continuity3} more compactly as 
\begin{align*}
\E|\nabla_w V(X_{t_-})|^2 \le \frac{2M_0}{\alpha}\big(2\mathsf{N}H_{t_-}(w) + \mathsf{N}S(w)\big).
\end{align*}
Plug this back into \eqref{pf:weak-continuity1} to get
\begin{equation}
A^1_t(u) \le C_t(u) := \frac{M_0}{1-\epsilon}\Big( \frac{2M_0}{\alpha}h^2\mathsf{N}^2H_{t_-}(u) + \frac{M_0}{\alpha}h^2\mathsf{N}^2S(u) + h\mathsf{N}S(u)\Big). \label{pf:continuityterm1-weak}
\end{equation}

\subsection{The hierarchy term, $A^2_t$}

Using the specific form of the potential and \eqref{ineq:CScombinatorial},
\begin{align*}
A^2_t(u) &= \frac{1}{2\epsilon}\sum_{i \in u}\int_{\R^u} \bigg(\sum_{w\ni i}\E[\partial_i V_w(X^w_t)\,|\,X^u_t=x_u] - \E_{Y \sim \pi}[\partial_i V_w(Y^w)\,|\,Y^u=x_u]\bigg)^2  \,\rho^u_t(dx_u) \\
	&\le \frac{M_0}{2\epsilon}\sum_{w \cap u \neq \emptyset}\frac{1}{L_w}\int_{\R^u}\Big|\E[\nabla_w V_w(X^w_t)\,|\,X^u_t=x_u] - \E_{Y \sim \pi}[\nabla_w V_w(Y^w)\,|\,Y^u=x_u]\Big|^2  \,\rho^u_t(dx_u).
\end{align*}
We argue as  in Section \ref{se:hierarchyterm-sparse}, writing in terms of conditional measures as
\begin{align*}
A^2_t(u) &\le \frac{M_0}{2\epsilon}\sum_{w \cap u \neq \emptyset} \frac{1}{L_w}\int_{\R^u} \big|\langle \rho^{w \setminus u|u}_{t,x_u}-\pi^{w \setminus u|u}_{x_u}, \,\nabla_w V_w(x_u,\cdot)\rangle\big|^2  \,\rho^u_t(dx_u)
\end{align*}
Since $\nabla_w V_w$ is $L_w$-Lipschitz, the assumed transport inequality (B.2) followed by the chain rule for relative entropy yield
\begin{align}
A^2_t(u) &\le \frac{ \gamma M_0}{\alpha\epsilon}\sum_{w \cap u \neq \emptyset} L_w \int_{\R^u} H\big(\rho^{w \setminus u|u}_{t,x_u}\,|\,\pi^{w \setminus u|u}_{x_u}\big)\,\rho^u_t(dx_u) \nonumber \\
	&= \frac{\gamma M_0}{\alpha\epsilon}\sum_{w \cap u \neq \emptyset} L_w\big(H_t(w \cup u) - H_t(u)\big). \label{pf:hierarchyterm1-weak}
\end{align}

\begin{remark}
The above paragraph is the only place that assumption (B.2) is used in the proof of Theorem \ref{th:weak}. It could thus be weakened, by instead directly assuming
\[
 \big|\langle \mu - \pi^{w \setminus u|u}_{x_u}, \,\nabla_w V_w(x_u,\cdot)\rangle\big|^2 \le \frac{2\gamma L_w^2}{\alpha}H(\mu\,|\,\pi^{w \setminus u|u}_{x_u}), \qquad \forall w,u \subset [n], \ x_u \in \R^u, \ \mu \in \P(\R^u).
\]
\end{remark}

\subsection{Developing the hierarchy}

To understand the resulting system of inequalities, we proceed as in Section \ref{se:developing-sparse}, but now with a more complicated Markov process.
View $H_t$ again as a vector indexed by sets, i.e., $H_t \in \R^{2^{[n]}}$. 
Define a linear operator $\mathsf{A} : \R^{2^{[n]}} \to \R^{2^{[n]}}$ by
\[
\mathsf{A} F(u) := \frac{\gamma M_0}{\alpha\epsilon}\sum_{w \cap u \neq \emptyset} L_w\big(F(w \cup u) - F(u)\big).
\]
Recognizing this operator in \eqref{pf:hierarchyterm1-weak}, we may thus combine Proposition \ref{pr:H_t(u)firstbound} with \eqref{pf:hierarchyterm1-weak} and \eqref{pf:continuityterm1-weak} to get the coordinatewise inequality
\[
\frac{d}{dt}H_t \le C_t + (\mathsf{A}-\alpha I)H_t.
\]
The operator $\mathsf{A}$ again has the structure of an infinitesimal generator (rate matrix) for a continuous-time Markov process. We will not use the stochastic representation here, preferring to work directly with operators. But it is important to note that the semigroup $e^{t\mathsf{A}}$ for $t > 0$ is monotone with respect to coordinatewise order: $e^{t\mathsf{A}}F  \le e^{t\mathsf{A}}G$ whenever $F \le G$.
Note also that $\mathsf{A}C_{(k-1)h} \ge 0$ coordinatewise because $C_{(k-1)h}$ is nondecreasing with respect to set inclusion, and this implies that $t\mapsto e^{t\mathsf{A}}C_{(k-1)h}(u)$ is nondecreasing.
We may now copy the argument leading to \eqref{pf:Hkh-bound1} to get 
\begin{equation}
H_{kh} \le e^{-\alpha h}e^{h\mathsf{A}}H_{(k-1)h} + \frac{1}{\alpha}(1-e^{-\alpha h}) e^{h\mathsf{A}} C_{(k-1)h}. \label{pf:Hkh-bound1-weak}
\end{equation}
Next, define a function $G \in \R^{2^{[n]}}$ by
\begin{equation}
G := \frac{h M_0}{1-\epsilon}\Big( \frac{M_0}{\alpha}h\mathsf{N}^2S(u) + \mathsf{N}S(u)\Big). \label{pf:Gdef-weak}
\end{equation}
This way, we may rewrite
\[
C_{(k-1)h} = G +  \frac{ 2h^2  M_0^2}{\alpha(1-\epsilon)} \mathsf{N}^2 H_{(k-1)h}.
\]
Plugging this into \eqref{pf:Hkh-bound1-weak} and combining the $H_{(k-1)h}$ terms, 
\begin{align*}
H_{kh} &\le \Big(e^{-\alpha h}e^{h\mathsf{A}} + \frac{2h^2 M_0^2}{\alpha^2 (1-\epsilon)}(1-e^{-\alpha h}) e^{h\mathsf{A}}\mathsf{N}^2 \Big)H_{(k-1)h} + \frac{1-e^{-\alpha h}}{\alpha} e^{h\mathsf{A}} G.
\end{align*}
Iterate this inequality to find
\begin{equation}
\begin{split}
H_{kh} \le \ & \Big(e^{-\alpha h}e^{h\mathsf{A}} + \frac{2h^2 M_0^2}{\alpha^2 (1-\epsilon)}(1-e^{-\alpha h}) e^{h\mathsf{A}}\mathsf{N}^2 \Big)^k H_0 \\
	&\qquad + \frac{1-e^{-\alpha h}}{\alpha} \sum_{j=0}^{k-1} \Big(e^{-\alpha h}e^{h\mathsf{A}} + \frac{2h^2 M_0^2}{\alpha^2 (1-\epsilon)}(1-e^{-\alpha h}) e^{h\mathsf{A}}\mathsf{N}^2 \Big)^je^{h\mathsf{A}} G .
\end{split} \label{pf:operator-iteration-weak}
\end{equation}
Here we must depart from the method of Section \ref{se:proof-sparse}, because the operators $\mathsf{N}$ and $\mathsf{A}$ do not necessarily commute. Instead, the key ingredients will be the following estimates for how these operators act on the size function $S(u)=|u|$. First, using \eqref{eq:combinatorial1},
\begin{align*}
\mathsf{A}S(u) &= \frac{\gamma M_0}{\alpha\epsilon} \sum_{w \cap u \neq \emptyset} L_w |w \setminus u| \\
	&= \frac{\gamma M_0}{\alpha\epsilon} \sum_{i \in u}\sum_{w \ni i} L_w \frac{|w \setminus u|}{|w \cap u|} \\
	&\le \frac{\gamma M_0}{\alpha\epsilon} \sum_{i \in u}\sum_{w \ni i } L_w(|w|-1) \\
	&\le \frac{\gamma M_0 R_1}{\alpha\epsilon}S(u).
\end{align*}
This quickly implies
\begin{align*}
e^{t\mathsf{A}}S(u) \le e^{\lambda t} S(u), \qquad t \ge 0, \qquad \lambda := \frac{\gamma M_0 R_1}{\alpha\epsilon}.
\end{align*}
A similar calculation shows
\begin{align*}
\mathsf{N}S(u) &= \sum_{w \cap u \neq \emptyset}L_w|w| \le  M_1S(u).
\end{align*}
Combining these last two bounds, we find  
\begin{equation}
\Big(e^{-\alpha h}e^{h\mathsf{A}} + \frac{2h^2 M_0^2}{\alpha^2 (1-\epsilon)}(1-e^{-\alpha h}) e^{h\mathsf{A}}\mathsf{N}^2 \Big) S \le \Big(e^{-\alpha h}  + \frac{2h^2 M_0^2M_1^2}{\alpha^2 (1-\epsilon)}(1-e^{-\alpha h}) \Big)  e^{\lambda  h} S. \label{pf:weak11}
\end{equation}
We simplify this using the real variable inequality \eqref{ineq:1-eps.x}.
Define
\begin{equation}
\widehat{h} := \min\bigg(\frac{2(1-\epsilon)}{\alpha}, \ \frac{\alpha}{\sqrt{2}M_0M_1}\bigg((1-\epsilon)\Big(1 - \frac{2 \epsilon(1-\epsilon)}{1-e^{-2(1-\epsilon)}}\Big)\bigg)^{1/2}\bigg). \label{def:h-hat-weak}
\end{equation}
For $h \le \widehat{h}$ we have
\[
e^{-\alpha h}  + \frac{2h^2 M_0^2 M_1^2}{\alpha^2 (1-\epsilon)}(1-e^{-\alpha h}) \le 1-\alpha\epsilon h \le e^{-\alpha\epsilon h}.
\]
Hence, we may simplify \eqref{pf:weak11} further:
\begin{equation}
\Big(e^{-\alpha h}e^{h\mathsf{A}} + \frac{2h^2 M_0^2 M_1^2}{\alpha^2 (1-\epsilon)}(1-e^{-\alpha h}) e^{h\mathsf{A}}\mathsf{N}^2 \Big) S \le   e^{(\lambda-\alpha\epsilon )  h} S.\label{pf:weak-iter11}
\end{equation}
Note that the weak interaction assumption (B.3) allows us to choose $\epsilon$ close to 1 to ensure that
\[
\tau := \alpha\epsilon - \lambda  = \alpha\epsilon - \frac{\gamma M_0 R_1}{\alpha\epsilon}  > 0.
\]
Noting that $H_0 \le C_0S$ by the assumption (Init), the bound \eqref{pf:weak-iter11} lets us estimate the first term of \eqref{pf:operator-iteration-weak} by
\begin{equation}
\Big(e^{-\alpha h}e^{h\mathsf{A}} + \frac{2h^2 M_0^2}{\alpha^2 (1-\epsilon)}(1-e^{-\alpha h}) e^{h\mathsf{A}}\mathsf{N}^2 \Big)^k H_0 \le C_0e^{-\tau kh  } S. \label{pf:weak-iter-term1}
\end{equation}
To handle the second term of \eqref{pf:operator-iteration-weak}, we combine $e^{t\mathsf{A}}S \le e^{\lambda t}S$ with \eqref{pf:weak-iter11} to obtain
\begin{equation*}
\sum_{j=0}^{\infty}\Big(e^{-\alpha h}e^{h\mathsf{A}} + \frac{2h^2 M_0^2}{\alpha^2 (1-\epsilon)}(1-e^{-\alpha h}) e^{h\mathsf{A}}\mathsf{N}^2 \Big)^je^{h\mathsf{A}} S \le \frac{e^{\lambda h}}{1 - e^{-\tau h}} S.
\end{equation*}
Recalling the definition of $G$ from \eqref{pf:Gdef-weak} and using $\mathsf{N}S \le M_1S$ lets us finally bound the second term of \eqref{pf:operator-iteration-weak} by
\begin{align*}
\frac{1-e^{-\alpha h}}{\alpha} &\sum_{j=0}^{k-1} \Big(e^{-\alpha h}e^{h\mathsf{A}} + \frac{2h^2 M_0^2}{\alpha^2 (1-\epsilon)}(1-e^{-\alpha h}) e^{h\mathsf{A}}\mathsf{N}^2 \Big)^je^{h\mathsf{A}} G \\
	&\le h \frac{e^{\lambda h}}{\alpha(1-\epsilon)}\frac{1-e^{-\alpha h}}{1 - e^{-\tau h}}  \bigg(\frac{hM_0^2M_1^2}{\alpha} + M_0M_1 \bigg) S.
\end{align*}
Plugging this and \eqref{pf:weak-iter-term1} into \eqref{pf:operator-iteration-weak}, and using also $\alpha^{-1}(1-e^{-\alpha h})/(1-e^{-\tau h}) \le \tau^{-1}+ h$, we  get
\begin{equation}
H_{kh}(u) \le C_0e^{-\tau kh}|u| +  h \frac{e^{\lambda h}}{1-\epsilon} \bigg(\frac{1}{\tau}+h\bigg) \bigg(\frac{ hM_0^2M_1^2}{\alpha } +   M_0M_1  \bigg) |u|. \label{pf:operator-iteration-weak2}
\end{equation}
This is valid for $h \le \widehat{h}$, where $\widehat{h}$ was defined in \eqref{def:h-hat-weak}. We have thus proven Theorem \ref{th:weak-dynamic}, except with $\widehat{h}$ instead of $h^*$, and with $C$ replaced by
\[
\widehat{C}_h =  \frac{e^{\gamma M_0 R_1 h/\alpha\epsilon}M_0M_1}{1-\epsilon}  \bigg(\frac{ M_0 M_1 h }{\alpha } +   1  \bigg)\bigg(\frac{1}{\tau} + h\bigg).
\]

\subsection{Simplifying the constants}
We finally show that $\widehat{h} \ge h^*$ and that $C \le \widehat{C}$. Take $\epsilon= \frac12 + \frac{\gamma M_0 R_1}{2\alpha^2}$, which is the midpoint between $\frac{\gamma M_0 R_1}{\alpha^2}$ and 1. Then the exponent $\tau$ becomes
\begin{align*}
\tau &= \frac{\alpha}{2} + \frac{\gamma M_0 R_1}{2\alpha}  - \frac{\gamma M_0 R_1}{\frac{\alpha}{2} + \frac{\gamma M_0 R_1}{2\alpha} } = \frac{\alpha}{2} \frac{\Big(1 - \frac{\gamma M_0 R_1}{\alpha^2} \Big)^2 }{1 + \frac{\gamma M_0 R_1}{\alpha^2} }.
\end{align*}
Recalling the definition $\eta = 1 - \frac{\gamma M_0 R_1}{\alpha^2}$ from the theorem statement, we get 
\begin{align*}
\tau = \frac{\alpha\eta^2}{2(2-\eta)}, \quad 1-\epsilon  &= \frac{\eta}{2}, \quad 2\epsilon(1-\epsilon) = \frac12\eta(2-\eta).
\end{align*} 
Using $\epsilon \ge 1/2$ in the exponent,
the constants then become
\begin{align*}
\widehat{C}_h &= \frac{2M_0 M_1}{\eta} e^{2\gamma M_0 R_1 h/\alpha}  \Big(\frac{ M_0 M_1 h}{\alpha } +  1  \Big)\bigg(\frac{4-2\eta}{\alpha\eta^2} + h\bigg), \\ 
\widehat{h} &= \min\bigg\{\frac{\eta}{\alpha}, \frac{\alpha\sqrt{\eta}}{2 M_0M_1  } \bigg(1-\frac{\frac12\eta(2-\eta)}{1-e^{-\eta} }\bigg)^{1/2}\bigg\} .
\end{align*}
To bound $\widehat{h}$, we again use the inequality \eqref{ineq:etabound-h} to deduce
\[
\widehat{h} \ge \min\bigg\{\frac{\eta}{\alpha}, \frac{\alpha \eta^{3/2}}{5  M_0 M_1  }  \bigg\} \ge \frac{\alpha \eta^{3/2}}{5  M_0 M_1  }  = h^*,
\]
where the last step used $0 < \eta < 1$ along with the fact that $\alpha$ is bounded by the Lipschitz constant of $\nabla V$, which is itself bounded by $M_0 \le M_1$.
The exponential in  $\widehat{C}_h$ is bounded by $e^{\gamma}$ for $h \le h^*$, because $h^* \le \alpha/5M_0M_1$ and $R_1 \le M_1$.
Using $h^* \le \eta/\alpha$ and also $4-2\eta+\eta^3\le 4$ since $0 \le \eta \le 1$, we get
\begin{align*}
\widehat{C}_h &\le \frac{2M_0 M_1}{\eta} e^{\gamma }  \cdot \frac65 \cdot \frac{4-2\eta+\eta^3}{\alpha\eta^2 }  \le \frac{10M_0 M_1e^{\gamma } }{\alpha\eta^3 }   = C.
\end{align*}

\section{Proof of Theorem \ref{th:onestepcontraction} via coupling} \label{se:onestepcontration}

Here we give the proof of Theorem \ref{th:onestepcontraction}. As a prelimary step, we claim that
\begin{equation}
\langle \pi,|\nabla V|_\infty^2\rangle \le 4\beta\log(2n). \label{pf:subgaussianestimate}
\end{equation} 
This improves on \cite[Proposition 2.3]{DelocalizationOfBias} by a factor of $\beta/\alpha$, and it allows us to require merely $h\le 1/\beta$ instead of  $h\le\alpha/\beta^2$. We thank Sinho Chewi for pointing out this improvement. Indeed, this comes from the fact that $\nabla V=-\nabla\log\pi$ is $\sqrt{\beta}$-subgaussian under $\pi$, in the sense that
\[
\log \int e^{\theta \cdot \nabla V(x)}\,\pi(dx) \le \frac{\beta}{2}|\theta|^2, \qquad \forall \theta \in \R^n.
\]
See \cite[Theorem 2.2]{negrea2022approximations} or \cite[Theorem 1.2]{altschuler2023shifted}.  Then, applying \cite[Lemma B.1]{DelocalizationOfBias} yields \eqref{pf:subgaussianestimate}.

Turning now to the main line of the proof, we consider the usual coupling of the LMC iterates with the continuous-time Langevin dynamics. That is, let $(B_t)_{t \ge 0}$ be an $n$-dimensional Brownian motion, and let $Y_0 \sim \pi$. Let
\begin{align*}
X_{(k+1)h} &= X_{kh} - h\nabla V(X_{kh}) - \sqrt{2h}(B_{(k+1)h}-B_{kh}), \qquad k \ge 0, \\
dY_t &= -\nabla V(Y_t)dt + \sqrt{2}dB_t.
\end{align*}
We introduce an auxiliary process as in \cite{DelocalizationOfBias}:
\begin{align*}
\overline{Y}_{(k+1)h} = Y_{kh}  - h\nabla V(Y_{kh}) - \sqrt{2h}(B_{(k+1)h}-B_{kh}), \qquad k \ge 0.
\end{align*}
By the triangle inequality,
\begin{equation}
\big(\E|X_{(k+1)h}-Y_{(k+1)h}|_\infty^2\big)^{1/2} \le \big(\E|X_{(k+1)h}-\overline{Y}_{(k+1)h}|_\infty^2\big)^{1/2} + \big(\E|\overline{Y}_{(k+1)h}-Y_{(k+1)h}|_\infty^2\big)^{1/2}. \label{pf:contraction1}
\end{equation}
Using $\nabla^2 V \le \beta I$, we may exactly follow the calculation of \cite[(A.6)--(A.7)]{DelocalizationOfBias} to control the second term by
\begin{equation*}
\E|\overline{Y}_{(k+1)h}-Y_{(k+1)h}|_\infty^2 \le \frac{2\beta^2 h^4}{3}\int_{\R^n}|\nabla V|_\infty^2 \,d\pi + 8\beta^2 h^3\log(2n),
\end{equation*}
Apply \eqref{pf:subgaussianestimate}, use $h\le 1/\beta$, and combining terms to get
\begin{equation}
\E|\overline{Y}_{(k+1)h}-Y_{(k+1)h}|_\infty^2 \le 16 \beta^2h^3 \log(2n). \label{pf:contraction2} 
\end{equation}
To deal with the first term of \eqref{pf:contraction1}, interpolate in the usual way and use the assumed structure of $\nabla^2V$:
\begin{equation*}
 X_{(k+1)h}-\overline{Y}_{(k+1)h} = X_{kh} - Y_{kh}  - h\big(\nabla V(X_{kh}) - \nabla V(Y_{kh})\big) = (R_D-hR_M)(X_{kh} - Y_{kh}),
\end{equation*}
where we define the random matrices
\begin{equation*}
R_D := I - h\int_0^1D(uX_{kh} + (1-u)Y_{kh})\,du, \qquad R_M := \int_0^1M(uX_{kh} + (1-u)Y_{kh})\,du.
\end{equation*}
Note that the entries of $D$ lie in the interval $[\alpha,\beta]$, and so the entries of the diagonal matrix $R_D$ lie in $[1-\beta h,1-\alpha h] \subset [0,1-\alpha h]$ since $h \le \alpha/\beta^2 \le  1/\beta$. Since $R_D$ is diagonal, we deduce $\|R_D\|_{\infty\to\infty} \le 1-\alpha h$. The norm bound assumed for $M$ thus yields
\begin{align*}
|R_D-hR_M|_{\infty\to\infty} &\le |1-\alpha h| + h|R_M|_{\infty\to\infty} \le 1-(\alpha-\alpha_0) h.
\end{align*}
It follows that
\[
\big(\E|X_{(k+1)h}-\overline{Y}_{(k+1)h}|_\infty^2\big)^{1/2} \le \big(1-(\alpha-\alpha_0) h\big)\big(\E|X_{kh} - Y_{kh}|_\infty^2\big)^{1/2}.
\]
Now, assuming $(X_{kh},Y_{kh})$ was coupled optimally for $W_{2,\ell^\infty}(\rho_{kh},\pi)$, we combine this estimate with \eqref{pf:contraction2} to deduce
\begin{align*}
W_{2,\ell^\infty}(\rho_{(k+1)h},\pi) &\le e^{-(\alpha-\alpha_0)h}W_{2,\ell^\infty}(\rho_{kh},\pi) +  4\beta h^{3/2}\sqrt{\log(2n)}.
\end{align*}
Iterate this to get the claimed bound on the bias. \hfill \qedsymbol

\section{Proof of Theorem \ref{th:continuoustime}} \label{se:continuoustime}

We start with a straightforward adaptation of Propositions \ref{pr:marginaldynamics}, noting that $(\rho^u_t)_{t \ge 0}$ solves the Fokker-Planck equation $\partial_t \rho^u =\nabla_u\cdot(b^u\rho^u)+\Delta_u\rho^u$ for $b^u_t(x_u) := \E[\nabla_uV(Y_t)\,|\,Y^u_t=x_u]$. Then, arguing as in the proof of \ref{pr:H_t(u)firstbound} up to \eqref{ineq:completesquare}, and applying Remark \ref{re:completesquare} to introduce the free parameter $\epsilon \in (0,1)$, we have
\begin{align*}
\frac{d}{dt}H_t(u) &\le \frac{1}{4\epsilon}\E\big|\nabla\log\pi^u(Y^u_t) + \E[\nabla_u V(Y_t)\,|\,Y^u_t]\big|^2 - (1-\epsilon)I(\rho^u_t\,|\,\pi^u) \\
	&= \frac{1}{4\epsilon}\E\big|\langle \rho^{N(u)\setminus u}_{t,Y^u_t} - \pi^{N(u)\setminus u}_{Y^u_t},\,\nabla_uV(Y^u_t,\cdot)\rangle\big|^2 - (1-\epsilon)I(\rho^u_t\,|\,\pi^u).
\end{align*}
Here we write $\rho^{N(u)\setminus u}_{t,Y^u_t}$ as usual for the conditional law of $Y^{N(u)}_t$ given $Y^u_t$. Apply the transport inequality (A.2) along with the chain rule for relative entropy, exactly as in Section \ref{se:hierarchyterm-sparse}, to get
\[
\frac{d}{dt}H_t(u) \le\frac{\gamma\beta^2}{2\alpha\epsilon}\big(H_t(N(u))-H_t(u)\big) - 2\alpha (1-\epsilon)H_t(u).
\]
Define the operator $\mathsf{A} : \R^{2^{[n]}} \to \R^{2^{[n]}}$ by $\mathsf{A}F(u)=\frac{\gamma\beta^2}{2\alpha}(F(N(u))-F(u))$. Then the above rewrites as a coordinatewise inequality between vectors,
\begin{equation*}
\frac{d}{dt}H_t \le \bigg(\frac{1}{ \epsilon}\mathsf{A}-2\alpha(1-\epsilon)\mathsf{I}\bigg)H_t.
\end{equation*}
Integrate this coordinatewise inequality as in  \ref{pf:Gronwall} to find
\begin{equation*}
H_t \le e^{-2\alpha(1-\epsilon)t} e^{ t\mathsf{A}/\epsilon}H_0.
\end{equation*}
As explained in the paragraph after \eqref{def:Aoperator-sparse},
the semigroup admits the stochastic representation $e^{s\mathsf{A}}F(u)=\E F(N_{\Lambda(s)}(u))$ for any $F \in \R^{2^{[n]}}$, where $\Lambda$ is a Poisson process with rate $\gamma\beta^2/2\alpha$. 
This completes the proof. \hfill\qedsymbol

\bibliographystyle{amsplain}
\bibliography{biblio}

\end{document}